\documentclass[11pt]{article}

\usepackage{comment,url,algorithm,algorithmic,graphicx,subcaption,relsize}
\usepackage{amssymb,amsfonts,amsmath,amsthm,amscd,dsfont,mathrsfs,mathtools,nicefrac}
\usepackage{float,psfrag,epsfig,color,xcolor,url,hyperref}
\usepackage{epstopdf,bbm,mathtools,enumitem}
\usepackage[toc,page]{appendix}
\usepackage[mathscr]{euscript}
\usepackage{soul}
\usepackage[top=1in, bottom=1in, left=1in, right=1in]{geometry}

\makeatletter
\@addtoreset{equation}{section}
\makeatother

\hypersetup{
  colorlinks,
  linkcolor={red!50!black},
  citecolor={blue!50!black},
  urlcolor={blue!80!black}
}

\newcommand{\norm}[1]{\ensuremath{\| #1 \|}}

\newcommand{\inprod}[2]{\ensuremath{\langle #1 , \, #2 \rangle}}

\ifdefined\nonewproofenvironments\else
\ifdefined\ispres\else

\newtheorem{theorem}{Theorem}[section]
\newtheorem{proposition}{Proposition}[section]
\newtheorem{lemma}{Lemma}[section]
\newtheorem{corollary}{Corollary}[section]
\newtheorem{remark}{Remark}[section]

\newtheorem{assumption}{Assumption}[section]

\newcommand{\bigo}{\mathcal{O}}

\newcommand{\RN}[1]{%
  \textup{\uppercase\expandafter{\romannumeral#1}}%
}

\newcommand{\1}{{\rm 1}\mskip -4,5mu{\rm l} }

\newcommand{\op}[1]{\operatorname{#1}}  %

\def\E{\mathbb{E}}

\def\R{\mathbb{R}}

\newcommand{\rp}{\mathbb{R}^p}
\newcommand{\rd}{\mathbb{R}^d}

\newcommand{\mpr}{\mathbb{P}}

\newcommand{\vecx}{\mathbf{x}}
\newcommand{\vecy}{\mathbf{y}}

\newcommand{\borel}{\mathcal{B}}
\newcommand{\kernel}{P}
\newcommand{\iter}{\theta^{(\eta)}}
\newcommand{\bariter}{\bar{\theta}_{\eta}}
\newcommand{\noise}{\xi}
\newcommand{\sdist}{\pi_\eta}
\newcommand{\chain}{\{\iter_k\}_{k\ge 0}}

\newcommand{\filtration}{\mathcal{F}}

\newcommand{\grad}{\nabla}

\mathtoolsset{showonlyrefs}

\def\balign#1\ealign{\begin{align}#1\end{align}}
\def\baligns#1\ealigns{\begin{align*}#1\end{align*}}
\def\balignat#1\ealign{\begin{alignat}#1\end{alignat}}
\def\balignats#1\ealigns{\begin{alignat*}#1\end{alignat*}}
\def\bitemize#1\eitemize{\begin{itemize}#1\end{itemize}}
\def\benumerate#1\eenumerate{\begin{enumerate}#1\end{enumerate}}

\newcommand{\eq}[1]{\begin{align}#1\end{align}}
\newcommand{\eqn}[1]{\begin{align*}#1\end{align*}}

\begin{document}

\title{An Analysis of Constant Step Size SGD in the Non-convex Regime: Asymptotic Normality and Bias}

 \author{
Lu Yu\thanks{
Department of Statistical Sciences at
 the University of Toronto, and Vector Institute~~\texttt{luyu@utstat.toronto.edu}
 }~~~~
Krishnakumar Balasubramanian\thanks{
Department of Statistics, University of California, Davis~~\texttt{kbala@ucdavis.edu}
 }~~~~
 Stanislav Volgushev\thanks{
Department of Statistical Sciences at
 the University of Toronto~~\texttt{stanislav.volgushev@utoronto.ca}
 }~~~~
 Murat A. Erdogdu\thanks{
  Department of Computer Science and Department of Statistical Sciences at
   the University of Toronto, and Vector Institute~~\texttt{erdogdu@cs.toronto.edu}
 }
}

\maketitle

\begin{abstract}
 Structured non-convex learning problems, for which critical points have favorable statistical properties, arise frequently in statistical machine learning. Algorithmic convergence and statistical estimation rates are well-understood for such problems. However, quantifying the uncertainty associated with the underlying training algorithm is not well-studied in the non-convex setting. In order to address this short-coming, in this work, we establish an asymptotic normality result for the constant step size stochastic gradient descent (SGD)  algorithm---a widely used algorithm in practice. Specifically, based on the relationship between SGD and Markov Chains~\cite{dieuleveut2017}, we show that the average of SGD iterates is asymptotically normally distributed around the expected value of their unique invariant distribution, as long as the non-convex and non-smooth objective function satisfies a dissipativity property. We also characterize the bias between this expected value and the critical points of the objective function under various local regularity conditions. Together, the above two results could be leveraged to construct confidence intervals for non-convex problems that are trained using the SGD algorithm.
 \end{abstract}

\section{Introduction}

Non-convex learning problems are prevalent in modern statistical machine learning applications such as matrix and tensor completion~\cite{ge2015escaping, ge2016matrix, xia2017statistically, chi2019nonconvex, cai2019nonconvex}, deep neural networks~\cite{goodfellow2016deep, jain2017non,maddox2019simple},
and robust empirical risk minimization~\cite{loh2017statistical,lecue2018robust, mei2018landscape}. Developing theoretically principled approaches for tackling such non-convex problems depends critically on the interplay between two aspects. From a computational perspective, variants of stochastic gradient descent (SGD)
converge to first-order critical points~\cite{ghadimi2013stochastic, fang2018spider} or
local minimizers~\cite{nesterov2006cubic, ge2015escaping, jin2017escape, tripuraneni2018stochastic}
of the objective function. From a statistical perspective, \emph{oftentimes} these critical points or local minimizers have nice statistical properties~\cite{kawaguchi2016deep, ge2016matrix,loh2017statistical, mei2017solving,elsener2018sharp,chi2019nonconvex}; see also \cite{freedman1982inconsistent} for a counterexample. For the purpose of uncertainty quantification in such non-convex learning paradigms, studying the fluctuations of iterative algorithms used for training becomes extremely important. In this work, we focus on the widely used constant step size SGD,
and develop results for quantifying the uncertainty associated with this algorithm for a class of non-convex problems.

\noindent Specifically, we consider minimizing a non-smooth and non-convex objective function $f\!:\rd\!\to\R$,
\eq{\label{eq:mainproblem}
\underset{{\theta\in\rd}}{\min}~f(\theta).
}
The iterations of SGD with a constant step size $\eta>0$, initialized at $\iter_0\equiv\theta_0\in\rd$,
are given by 
{
\eq{\label{ref:sgd_org}
	\iter_{k+1} = \iter_{k} - \eta \big(\grad f(\iter_k) +\noise_{k+1}(\iter_k)\big),~~~k\ge 0\,,
}
where $\{\noise_k\}_{k\ge 1}$ is a sequence of random functions from $\rd$ to $\rd$} corresponding
to the stochasticity in the gradient estimate. Several problems in machine learning and statistics are naturally formulated as the optimization problem is~\eqref{eq:mainproblem}, where the function $f(\theta)$ is given by 
\eq{\label{eq:onlinesgd_main}
f(\theta) \coloneqq \E_Z [F(\theta, Z)]= \int F(\theta, Z) \, dP(Z),}
 where the function $F(\theta, Z)$ is typically the loss function composed with functions from hypothesis class parametrized by $\theta \in \mathbb{R}^d$, and depends on the random variable $Z \in \mathbb{R}^p$. The distribution $P(Z)$ is typically unknown. Then the iterations of online SGD with a constant step size $\eta>0$, and batch-size $b_k$, initialized at $\iter_0\equiv\theta_0\in\rd$,
are given by 
\eq{\label{ref:sgd_org_online}
	\iter_{k+1} = \iter_{k} - \frac{\eta}{b_k}\sum_{j=1}^{b_k}\nabla F(\iter_k, Z_j),~~~k\ge 0\,,
}
where independent samples $Z_j\sim P(Z)$, is used to estimate the true gradient in each iteration $k$. Furthermore, the samples $Z_j$ used across all iterates $k$ are also independent. Typically, we also have $\nabla F(\theta, Z)$ to be unbiased estimates of the true gradient $\nabla f(\theta)$, for all $\theta \in \mathbb{R}^d$. The above iterates are indeed a special case of the iterates in~\eqref{ref:sgd_org}, with the noise sequence $\{\noise_k(\iter_k)\}_{k\ge 1}$ given by 
\eq{
\noise_k(\iter_k):= \frac{1}{b_k}\sum_{j=1}^{b_k} \left[\nabla F(\iter_k, Z_j) - \nabla f(\iter_k)\right].
}

Although proposed in the  1950s by~\cite{robbins1951stochastic}, SGD has been the algorithm of choice for training statistical models due to its simplicity,
and superior performance in large-scale
settings~\cite{fort1999asymptotic, dieuleveut2017,wilson2017marginal,balles2018dissecting}.
However, the fluctuations of this algorithm is well-understood
only when the objective function $f$ is strongly convex and smooth,
and the step size $\eta$ satisfies a specific decreasing schedule %
so that the iterates asymptotically converge to the \emph{unique} minimizer~\cite{polyak1992acceleration,duchi2016local,anastasiou2019normal}.
On the other hand, it is well-known that the SGD iterates in~\eqref{ref:sgd_org} can be viewed as a Markov chain which allows them to converge to a random vector rather than a single critical point~\cite{dieuleveut2017}.
Building on this analogy between SGD and Markov chains, the aforementioned shortcomings can be alleviated by simply relaxing  the global smoothness as well as the strong convexity assumptions to the tails of the objective function $f$, which allows for non-convex structure around the region of interest. Similar kinds of tail relaxations have been successfully employed in the diffusion theory
when the target potential is non-convex~\cite{raginsky2017non,cheng2018sharp,erdogdu2018global},
but they are not studied in the context of non-convex optimization with the SGD algorithm. In this work, we study the fluctuations and the bias of the averaged SGD iterates in~\eqref{ref:sgd_org}, around the first-order critical points of the minimization problem~\eqref{eq:mainproblem}. Our contributions can be summarized as follows.

\begin{itemize}%
	\item  For a non-convex and non-smooth objective function $f$ with tails growing at least quadratically, we establish the uniqueness of the stationary distribution of the constant step size SGD iterates in Proposition~\ref{ergodicity}, and the asymptotic normality of Polyak-Ruppert averaging in Theorem~\ref{general_clt_arb_sdist}.  To the best of our knowledge, these are the first uniqueness and normality results for the SGD algorithm when the objective function is non-convex (even not strongly convex) and non-smooth.

	\item We further show in Proposition~\ref{bias_wo_local} that, under the assumptions leading to the CLT, the asymptotic bias between the expectation of the Lipschitz test function~$\phi$ under the stationary distribution of the SGD iterates and the value of~$\phi$ at any first-order critical point
	is bounded by a constant depending on the tail growth properties of $f$.
	
	\item Finally, we show in Theorems~\ref{bias_wz_disp} and \ref{bias_wz_loja} that with additional local smoothness assumptions on the function~$f$ that allow non-convexity, we can establish a control over the bias in terms of step size.
	We further characterize the bias when the objective is (not strongly) convex in Theorem~\ref{bias_wz_convex}, 
	providing a thorough bias analysis for the constant step size SGD under various settings that are frequently encountered in statistics.
\end{itemize}

Our results provide algorithm-dependent guarantees for uncertainty quantification, and they could be potentially leveraged to obtain
confidence intervals for non-convex and non-smooth learning problems. This is contrary to the majority of the existing results in statistics, which only establish normality results for the true stationary point of the non-convex objective function; see for example~\cite{loh2017statistical, qi2019statistical}. While being useful, such results completely ignore the computational hardships associated with non-convex optimization; hence, their practical implications are limited. On the other hand, in the optimization and learning theory literature, a majority of the existing results establish the rate of convergence of an algorithm to a critical point, and do not quantify the fluctuations associated with that algorithm. 
Our work bridges these separate lines of thought by providing asymptotic normality results directly for the SGD algorithm used for minimizing non-convex and non-smooth functions.

\vspace{0.1in}

\textbf{More Related Works.} Establishing asymptotic normality results for the SGD algorithm began with the works of~\cite{chung1954stochastic, sacks1958asymptotic, fabian1968asymptotic, ruppert1988efficient, shapiro1989asymptotic}, with~\cite{polyak1992acceleration} providing a definitive result for strongly convex objectives. In particular, \cite{polyak1992acceleration} and~\cite{ruppert1988efficient} established that the averaged SGD iterates with an appropriately chosen decreasing step size is asymptotically normal with the variance achieving the Cramer-Rao lower bound for parameter estimation. Recent works, for example~\cite{tripuraneni2018averaging, su2018statistical, duchi2016local, toulis2017asymptotic,fang2018online}, leverage the asymptotic normality analysis of~\cite{polyak1992acceleration}, and compute confidence intervals for SGD. The benefits of constant step size SGD for faster convergence under overparametrization has also be demonstrated in the works of~\cite{schmidt2013fast, needell2014stochastic, ma2018power, vaswani2019fast}. The use of Markov chain theory to study constant step size stochastic approximation algorithms has been considered in several works~\cite{kifer1988, benaim1996dynamical, priouret1998remark, fort1999asymptotic, aguech2000perturbation, tan2019online}. Recently,~\cite{dieuleveut2017, chee2018convergence} investigated the asymptotic variance of constant step size SGD. We emphasize here that most of the above works assume strongly convex and smooth objective functions. Finally, there exists a vast literature on analyzing Markov chain Monte Carlo sampling algorithms based on discretizing diffusions. We refer the interested reader to~\cite{dalalyan2017user, brosse2017sampling, cheng2018sharp, durmus2017nonasymptotic, dalalyan2017theoretical, cheng2017underdamped, bubeck2018sampling, dwivedi2018log, dalalyan2018sampling, li2019stochastic,shen2019randomized,erdogdu2020convergence} and the references therein, for details.
\vspace{0.1in}

\textbf{Notation.} For $a, b\in\R$, denote by $a\vee b$ and $a\wedge b$ the maximum and the minimum of $a$ and $b$, respectively. 
We use $\norm{\cdot}$ to denote the Euclidean norm in $\mathbb{R}^d$. 
We denote the largest eigenvalue of the matrix~$A$ as~$\lambda_{\max}(A),$ and the smallest one as ~$\lambda_{\min}(A).$
Let~$(\Omega,\mathcal{F},\mpr)$ represent a probability space, and
denote by $\borel(\rd)$, the Borel~$\sigma$-field of $\rd$.
Let $\mathcal{P}_k(\rd) := \{\nu: \int_{\rd} \norm{\theta}^k\nu(d\theta)<\infty\}$ denote the set of probability measures with finite $k$-th moments.
For a probability distribution~$\pi$ and a function $g$ on $\mathcal{X}$, we define $\pi(g) \coloneqq \int_{\mathcal{X}} g(x) d\pi(x)$, and $\mathcal{L}_2(\pi):=\{g: \mathcal{X}\to\R: \pi(g^2) <\infty\}.$

\section{Central Limit Theorem for The Constant Step Size SGD}\label{sec:cltheory}

In this section, we establish an asymptotic central limit theorem (CLT) for the Polyak-Ruppert averaging of 
the constant step size SGD iterates given in \eqref{ref:sgd_org} 
when the objective function is potentially non-convex, non-smooth, and has quadratically growing tails.
More specifically, we first prove that there exists a unique stationary distribution~$\sdist\in\mathcal{P}_2(\rd)$ for the Markov chain defined by the SGD algorithm
when the objective function is dissipative (see Assumption~\ref{asm:dissipativity}) with gradient exhibiting at most linear growth (see Assumption~\ref{asm:growth}). 
Furthermore, under the same conditions, we prove that a CLT holds for 
the Polyak-Ruppert averaging, and it is independent of the initialization. In what follows, we list and discuss the main assumptions required to establish a CLT for the SGD iterates, and compare them to those existing in the literature.
\begin{assumption}[Linear growth]\label{asm:growth}
	The gradient of the objective function $f$
	has at most linear growth. That is, for some $L\ge 0$, we have 
	\eq{
	\norm{\grad f(\theta)}\le L\big(1+\norm{\theta}\big)\ \ \text{ for all }\ \ \theta\in\rd.
	}
\end{assumption}
Majority of the results on SGD focus on smooth functions with  gradients satisfying 
$\| \nabla f(\theta) - \nabla f(\theta') \| \leq \|\theta - \theta'\|$ 
for all $\theta, \theta' \in \mathbb{R}^d$; see e.g.~\cite{polyak1992acceleration,dieuleveut2017}.  The above condition allows for non-smooth objectives, and is a significant relaxation of the standard Lipschitz gradient condition. 
\begin{assumption}[Dissipativity]\label{asm:dissipativity}
	The objective function~$f$ is $(\alpha,\beta)$-dissipative. That is, 
	there exists positive constants~$\alpha,\beta$ such that
\eq{
	\inprod{\theta}{\grad f(\theta)}\ge \alpha\,\norm{\theta}^2-\beta\ \ \text{ for all} \ \ \theta\in\rd.
	}
\end{assumption}
The dissipativity assumption has its origins in the analysis of dynamical systems, and is used widely in the analysis of optimization and learning algorithms~\cite{mattingly2002,raginsky2017non, erdogdu2018global, xu2018global}. It could be viewed as a relaxation of strong convexity since it restricts the quadratic growth assumption to the tails of the function $f$, enforcing no local growth around the first-order critical points. 
A canonical example for this condition is the sum of a quadratic and any non-convex function with bounded gradient. 
For example, consider the function $x\to x^2+10\sin(x)$ which is clearly non-convex and $(1,25)$-dissipative. 
It is worth mentioning that many statistical learning problems 
such as phase retrieval~\cite{tan2019online} satisfy Assumption~\ref{asm:dissipativity}. %

\begin{assumption}[Noise sequence]\label{asm:noise}
Gradient noise sequence~$\{\noise_k\}_{k\ge 1}$ is a collection of i.i.d. random fields satisfying  
\eq{
\E[\noise_{1}(\theta)]=0\, \ \text{and } \  \E^{1/2} [\norm{\noise_{1}(\theta)}^2]\le L_{\noise}(1+\norm{\theta})\,,
}
for any $\theta\in\rd$ and a positive constant $L_{\noise}$. Moreover, for each $\theta \in \rd$ the distribution of the random variable $\noise_1(\theta)$ can be decomposed as $\mu_{1,\theta} + \mu_{2,\theta}$ where $\mu_{1,\theta}$ has a density, say $p_\theta$, with respect to Lebesgue measure which satisfies $\inf_{\theta \in C} p_{\theta}(t) > 0$ for any bounded set $C$ and any $t \in \rp$.
\end{assumption}
Assumption~\ref{asm:noise} as formulated above is stronger than what is used in the proofs. It can easily be seen that the lower bound on the density $p_\theta$ is only required to hold for a specific set whose form depends on $\eta$ and various constants from Assumptions~\ref{asm:growth}--\ref{asm:noise}. The form of this set is complicated, and an exact expression is given in the Appendix -- see equation~\eqref{ref:smallset}. We also emphasize that Assumption~\ref{asm:noise} also does not specify any explicit parametric form for the distribution of the noise sequence contrary to recent works in non-convex settings where dissipitavity condition has been heavily utilized~\cite{raginsky2017non,xu2018global,erdogdu2018global}.

We now establish the existence and uniqueness of the stationary distribution of the SGD iterates~\eqref{ref:sgd_org}.
\begin{proposition}[Ergodicity of SGD]
	\label{ergodicity}
	Let the Assumptions~\ref{asm:growth}-\ref{asm:noise} hold.
	For a step size satisfying
	$$
	\eta<\frac{\alpha-\sqrt{\bigl(\alpha^2-{(3L^2+L_\noise)} \bigr)\vee 0 }}{{3L^2+L_\noise}}\,,
	$$
	the following statements hold for the SGD~\eqref{ref:sgd_org}.

	\begin{enumerate}[label=(\alph*)]%
		\item \label{exist_sdist} SGD iterates admit a unique stationary distribution~$\sdist\in\mathcal{P}_2(\rd)$, depending on the choice of step size $\eta$. 
		\item \label{conv_to_sdist} 
		For a test function $\phi:\rd\to \R$ satisfying 
		$|\phi(\theta)| \le L_{\phi}(1+\norm{\theta})$ $\forall \theta\in\rd $ and some $L_\phi>0$,
		and for any initialization $\theta_0^{(\eta)}=\theta_0\in\rd$ of the SGD algorithm, there exists $\rho\in(0,1)$
		and $\kappa$ (both depending on $\eta$) such that we have 
		$$\big|\E\big[\phi\big(\theta_k^{(\eta)}\big)\big]-\sdist(\phi)\big| \le \kappa\, \rho^{k} (1+\norm{\theta_0}^2),$$  where $\sdist(\phi) \coloneqq \int \phi(x) d\sdist(x)$.
	\end{enumerate}
\end{proposition}
The uniqueness of the stationary distribution of the constant step size SGD has been established in~\cite{dieuleveut2017} for strongly convex and smooth objectives. 
In Proposition~\ref{ergodicity}, we relax both of these assumptions allowing for non-convex and non-smooth objectives. 
Our proof relies on $V$-uniform ergodicity~\cite{meyn2012}, which is fundamentally different from the ergodicity analysis in~\cite{dieuleveut2017}.
Under the dissipativity condition (quadratic growth of $f$), geometric ergodicity in Proposition~\ref{ergodicity} is not surprising; 
yet, it is worth highlighting that the function $f$ as well as the noise sequence require significantly less 
structure than what was assumed in the literature. 
The above step size assumption is almost standard
and it is required to obtain a uniform bound on the moments of SGD iterates. 
We highlight that similar to the gradient descent algorithm, the step size depends on a quantity that serves as a \emph{surrogate} condition number in our setting, namely, $L/\alpha$.
For the purposes of establishing a CLT, it is sufficient to consider moments of order 4 (in fact any order larger than 2 suffices), but it is also worth 
noting that any order moments of SGD can be controlled under Assumption~\ref{asm:dissipativity} 
as long as the noise has the same order finite moment.

Next, we state our first principal contribution, a central limit theorem for the averaged SGD iterates starting from any initial distribution for a non-convex objective. For a test function~$\phi:\rd\to\R$, we denote the centered partial sums of $\phi$ evaluated at the SGD iterates with $S_n(\phi)$, i.e., we define 
\eq{S_n(\phi)\coloneqq \sum_{k=0}^{n-1}\Big[\phi\big(\iter_{k}\big) - \pi_\eta(\phi)\Big]\ \ \text{ where }\ \ \sdist(\phi) \coloneqq \int \phi(x) d\sdist(x).}%

\begin{theorem}[CLT]
	\label{general_clt_arb_sdist}
	Let the Assumptions~\ref{asm:growth}-\ref{asm:noise} hold. 
	For a step size $\eta$ and a test function $\phi$ satisfying the conditions in Proposition~\ref{ergodicity}, we define $\sigma^2_{\sdist}(\phi)\coloneqq \lim_{n\to\infty} \frac{1}{n}\E_{\sdist} \bigl[S_n^2(\phi)\bigr]$.
	Then,
	\eq{
		n^{-1/2}S_n(\phi) \overset{\text{d}}{\longrightarrow} \mathcal{N}\big(0,\sigma^2_{\sdist}(\phi)\big)\,.
	}
\end{theorem}

The above result characterizes the fluctuations of a test function $\phi$
averaged across SGD iterates, even when the objective function is both non-convex and non-smooth. 
The asymptotic variance in the above CLT can be equivalently stated in another compact form.  
	If we define the centered test function as
	$h( \theta) = \phi(\theta)- \sdist(\phi)$,
	the asymptotic variance can be written as
        \eq{
        \sigma^2_{\sdist}(\phi)= 2\sdist(h\hat h)-\sdist(h^2)\ \ \text{ where }\ \ \hat h = \sum_{k=0}^\infty  \E\Big[h\big(\theta_k^{(\eta)}\big)\Big].
        }
	Indeed, this is the variance we compute at the end of our proof in Section~\ref{sec:cltproofs}. 
	However, the expression in Theorem~\ref{general_clt_arb_sdist} is obtained by
	simply applying \cite[Thm 21.2.6]{douc2018}.
	For the case of strongly convex functions with decreasing step size schedule, 
	it is well-known from the works of~\cite{polyak1992acceleration,ruppert1988efficient} 
	that the limiting variance of the averaged SGD iterates 
	achieves the Cramer-Rao lower bound for parameter estimation; 
see also~\cite{moulines2011non, anastasiou2019normal} for non-asymptotic rates in various metrics. 
The question of providing lower bounds for the limiting variance of the critical points in the non-convex setting is extremely subtle, and is often handled on a case-by-case basis. 
We refer the interested reader to~\cite{geyer1994asymptotics,shapiro2000asymptotics,loh2017statistical}.

There are several important implications of the above CLT result for constructing confidence intervals in practice. 
First note that, following the standard construction in inference, one can write the distribution of the sample mean approximately as 
$n^{-1}S_n(\phi) \approx \mathcal{N}\bigl(0,n^{-1}\sigma^2_{\sdist}(\phi) \bigr)$. Here, one needs to estimate the population quantity, 
the asymptotic variance  $\sigma^2_{\sdist}(\phi)$, for the purpose of obtaining confidence intervals. In Section~\ref{sec:discussion}, we discuss three strategies for estimating this quantity, 
which could be eventually used for inference in practice. A theoretical analysis of the proposed approaches in Section~\ref{sec:discussion} is beyond the scope of this work.

\section{Bias of the Constant Step Size SGD}\label{sec:bias}

In this section, we present a thorough analysis of the bias %
of constant step size SGD algorithm. %
We first show in Section~\ref{sec:wo-local-reg} that, in the non-convex and non-smooth case for which we established the CLT, 
the SGD algorithm converges to a ball that contains 
all the first-order critical points exponentially fast; nevertheless, 
the bias is not controllable with the step size.
Motivated by this, we provide three types of bias analyses in Section~\ref{BiasControl} under
different local growth assumptions on the objective~$f$,
characterizing the bias behavior in various non-convex and convex settings.
 {For this, we strength the assumption of the noise sequence as follows.
\begin{assumption}[Noise sequence]\label{asm:noise_bias}
Gradient noise sequence~$\{\noise_k\}_{k\ge 1}$ satisfies Assumption~\ref{asm:noise}, and it holds %
for any $\theta\in\rd$ that
\eq{
 \E\bigl[\norm{\noise_{1}(\theta)}^4 \bigr]\le L_{\noise}(1+\norm{\theta}^4)\,,
}
where $L_{\noise}$ is the same as in Assumption~\ref{asm:noise}.
\end{assumption}

}
\subsection{Bias without Local Regularity}\label{sec:wo-local-reg}

Bias behavior of an algorithm is intimately related to the local properties of the objective at critical points.
Therefore, under the mild assumptions that yield the CLT, one cannot expect a tight control over the bias.
However, the tail growth condition is sufficient for a rough characterization,
which is still important
because even the points that are close to the local minimizers generally have favorable computational~\cite{boumal2016non,mei2017solving,chi2019nonconvex}, 
and statistical properties~\cite{loh2017statistical, elsener2018sharp, qi2019statistical}.

If Assumption~\ref{asm:dissipativity} holds for an objective function $f$, 
all first-order critical points of $f$ must lie inside a ball of radius $\sqrt{\beta/\alpha}$.
Based on this, we show that the SGD iterates~\eqref{ref:sgd_org} will move towards this ball exponentially fast, 
which ultimately establishes a bound on the non-asymptotic bias, and in the limit case yields a bound on the asymptotic bias.
The following result formalizes this statement.
\begin{proposition}\label{bias_wo_local}
	Let Assumptions~\ref{asm:growth},\ref{asm:dissipativity}, and~\ref{asm:noise_bias} hold. For $\theta^*$ denoting an arbitrary critical point of the objective function $f$, define the constants~$\bar L:=L(1+\norm{\theta^*}),$ and 
\begin{align}
\label{eq:constants}
\begin{split}
	c_{L,\alpha} &\coloneqq \Bigl[\alpha-\sqrt{ \bigl(\alpha^2- {(3L^2+L_\noise)} \bigr)\vee 0 }\,\Bigr] \big/[ {3L^2+L_\noise}] \\
	c_{L,\alpha}^{\dagger} & \coloneqq \Bigl [\alpha-\sqrt{ (\alpha^2- {16L_\dagger})\vee 0 }\,\Bigr] \big/{(64  {L_\dagger})}
\end{split}
\end{align}
with $ {L_\dagger \coloneqq  \bar L^2+ 16\Bigl( L_\noise^{3/4}\bigl(1+(\beta/\alpha)^3 \bigr)\vee 
L_\noise^{1/2}\bigl(1+(\beta/\alpha)^2 \bigr)\vee
L_\noise\bigl (1+(\beta/\alpha)^4 \bigr)
\Bigr) .}$
	Then, for SGD iterates %
	initialized at a fixed point $\theta_0 \in\rd$ 
	and a step size satisfying 
	$\eta <1
	 {\wedge \frac{1}{10\bar L} }
	\wedge c_{L,\alpha}
	\wedge c_{L,\alpha}^\dagger\,,
	$
	we have
	\eq{\label{eq:nonasymp-bias}
		\E\bigl[\,\norm{\iter_{k}-\theta^*}^4 \bigr]^{1/2} \le {\rho}^{\, k} \,\norm{\theta_0-\theta^*}^2 + D\,,
	}
	where constants are \eq{
	D \coloneqq &\frac{64}{\alpha} \Bigl({\bar L^4+  {L_\noise \bigl(1+(\beta/\alpha)^4 \bigr) +512\bar L^6 +23 L_\noise^{3/2} \bigl(1+(\beta/\alpha)^6 \bigr) }  }\Bigr)^{1/2}\\
	&\vee \frac{8}{\alpha}\Bigl(\beta+(\sqrt{\alpha}+2L/\sqrt{\alpha})^2 \norm{\theta^*}+L\norm{\theta^*}+6\bar L^2+ {9L_\noise^{1/2} \bigl(1+ (\beta/\alpha)^2 \bigr )}+16\Bigr),\\
	\rho \coloneqq& \sqrt{1-2\alpha\eta+32 {L_\dagger}\eta^2}\in(0,1).} %
	Consequently, for any test function $\phi$ that is $L_\phi$-Lipschitz continuous, 
	we have %
	\eq{
	\bigl |\sdist(\phi)-\phi(\theta^*) \bigr| \le L_\phi \sqrt{D}\, .
		}
\end{proposition}

The above theorem establishes that the SGD algorithm initialized far away from any critical point
will converge (in the 4-th expectation) to the ball that contains all the first-order critical points exponentially fast.
The first term in the upper bound \eqref{eq:nonasymp-bias} depends on the initialization, but decays to zero exponentially fast with the number of iterations, for a fixed step size.
The second term in the bound \eqref{eq:nonasymp-bias} is a constant independent of the iteration number as well as the step size,
which serves as the squared radius of the ball that contains all the critical points plus an additional offset to account for the randomness in the SGD iterates.
In other words, SGD algorithm initialized at any point and with any sufficiently small step size 
will find this ball of interest exponentially fast.

\subsection{Bias with Local Regularity}\label{BiasControl}

In this part, we present algorithmically controllable bounds on the bias under local regularity conditions. 
Section~\ref{sec:gen-dissipativity} provides a direct control on $\E[\norm{\iter_k-\theta^*}]$ under the assumption that 
the unique minimizer $\theta^*$ exists. In Sections~\ref{sec:loja} and \ref{sec:cvx}, 
we characterize the degree of sub-optimality $\E[f(\iter_k)]-f^*$ where $f^*$ is the global minimum
which is not necessarily attained at a unique point.

\subsubsection{Localized dissipativity condition}\label{sec:gen-dissipativity}

We now introduce the generalized dissipativity condition which,
in addition to the tail growth enforced in Assumption~\ref{asm:dissipativity}, 
imposes a local growth around the unique critical point~$\theta^*$. %
\begin{assumption}[Localized dissipativity]\label{asm:dissipativity2}
	The objective function $f$ satisfies
	\eq{
		\inprod{\grad f(\theta)}{\theta-\theta^*}\ge 
		\begin{cases}
			\alpha\norm{\theta-\theta^*}^2-\beta & \norm{\theta-\theta^*}\ge R \\ 
			g\bigl(\norm{\theta-\theta^*}\bigr)& \norm{\theta-\theta^*}< R\,,
		\end{cases}
	}
	where $\theta^*\in\rd$ is the unique minimizer of $f$, $R:= \frac{\delta}{\alpha}+\sqrt{\frac{\beta}{\alpha}}$ with $\delta\in(0,\infty)$,
	$g:[0,\infty)\to [0,\infty)$ is a convex function with $g(0)=0$ whose inverse exists. 
\end{assumption}
If $g(x) = x^2$, the objective function is \textit{locally} strongly convex. 
However, the above assumption covers a wide range of objectives with different local growth rates depending on the function $g$.
Next, we show that the above assumption along with the assumptions leading to the CLT
is sufficient to establish an algorithmic control over the bias with a sufficiently small step size.
\begin{theorem}\label{bias_wz_disp}
	Let the Assumptions~\ref{asm:growth}, \ref{asm:noise_bias}, and~\ref{asm:dissipativity2} hold. 
	Then SGD iterates with step size satisfying $\eta < c_{L,\alpha}$ for $c_{L,\alpha}$ in~\eqref{eq:constants}
	admit the stationary distribution $\iter\sim\sdist$ which satisfies 
	\eq{
		\E\bigl[\norm{\iter-\theta^*}\bigr]\le   \frac{C}{\delta}\eta+ g^{-1}( C\eta),
	}
	where %
	\eq{
	C\coloneqq  {2\Bigl(3L^2+3L_\noise^{1/2}(1+(\beta/\alpha)^2)\Bigr)}\Bigl (\int\norm{\theta}^2\sdist(d\theta) +\norm{\theta^*}^2 \Bigr)+3L^2\norm{\theta^*}^2 +5L^2
	+{ 2L_\noise^{1/2}\Bigl(1+(\beta/\alpha)^2\Bigr)}\,.
	}
	Further, for a test function $\phi:\rd\to\R$ that is $L_\phi$-Lipschitz, the bias satisfies
	\eq{
	\bigl |\sdist(\phi)-\phi(\theta^*)\bigr|  \leq L_\phi\big({C}\eta/\delta+ g^{-1}( C\eta)\big).
		}
\end{theorem}
If the local growth is linear, i.e. $g(x)=x$, we obtain the bias $|\sdist(\phi)-\phi(\theta^*)|\leq \bigo(\eta).$
If local growth is quadratic, i.e. $g(x) = x^2$, the growth is \emph{locally} slower than the linear case; thus, 
we get bias control $|\sdist(\phi)-\phi(\theta^*)| \leq \bigo(\eta^{1/2})$, 
which is worse in step size dependency,
it reduces to the bound derived in~\cite[Lemma 10]{dieuleveut2017}.

We highlight that \cite{ding2019error} prove the following lower bound: $ \lim\inf_{k\to\infty}\E[\norm{\iter_k-\theta^*}^2]^{1/2} \geq c \eta^{1/2}$ for some $c>0$ under the assumption of Lipschitz gradients. This is in line with our findings since Lipschits gradients imply $g(x) \leq x^2$ for small $x$.

\subsubsection{Generalized \L{}ojasiewicz condition} \label{sec:loja}

In this section we work with a generalization of the \L{}ojasiewicz condition.
\begin{assumption}[Generalized \L{}ojasiewicz condition]\label{asm:loja}
	The objective function $f$ has a critical point $\theta^*$ and it satisfies
	\eq{
		\norm{\grad f(\theta)}^2\ge 
		\begin{cases}
			\gamma\bigl \{ f(\theta)-f(\theta^*)\bigr\}& \norm{\theta-\theta^*}\ge R, \\ 
			g\bigl(f(\theta)-f(\theta^*)\bigr)& \norm{\theta-\theta^*}< R,
		\end{cases}
	}
	where $\gamma$ and $R$ are positive constants, and $g:[0,\infty)\to [0,\infty)$ is a convex function with $g(0)=0$ whose inverse exists. 
\end{assumption}
In the case $g(x)=x^{\kappa}$ with $\kappa\in[1,2),$ for example, 
the above condition is termed as the \L{}ojasiewicz inequality~\cite{gao2016ojasiewicz}, 
and for $\kappa=1$, it reduces to the well-known Polyak-\L{}ojasiewicz (PL) inequality~\cite{karimi2016linear}. 
Note that this inequality implies that every critical point is a global minimizer; 
yet, it does not imply the existence of a unique critical point. 

The following result establishes an algorithmically controllable bias bound in terms of the step size.
\begin{theorem}\label{bias_wz_loja}
	Let the Assumptions~\ref{asm:growth},\ref{asm:dissipativity},~\ref{asm:noise_bias}, and~\ref{asm:loja} hold, and the Hessian
satisfies $\norm{\grad^2 f(\theta)}\le \tilde L(1+\norm{\theta}),$  $\forall \theta\in\rd\,$ and some $\tilde L$. 
Then, the SGD iterates with a step size satisfying $\eta < \frac{2}{\tilde L}\wedge c_{L,\alpha} \wedge  c_{L,\alpha}^{\dagger}\wedge 1$ for $c_{L,\alpha},c_{L,\alpha}^{\dagger}$ in~\eqref{eq:constants}
have the stationary distribution $\sdist$,
	\eq{
\sdist (f)-f(\theta^*) \le  g^{-1}\Bigl( \frac{2 M \eta}{2-\tilde L\eta} \Bigr) +\frac{2 M \eta}{2-\tilde L\eta}\,,
	}
	{
	where  \eq{
	 M\coloneqq&12\tilde L \bigl(L + L_\noise^{1/2}+L_\noise^{1/4}\bigr)^2\Bigl( 1 + m+m^{3/4}+\int\norm{\theta}^2\sdist(d\theta)\Bigr)\ \text{ with }\\
	m\coloneqq &\frac{8}{7\alpha}\Bigl[\bigl(\beta+6 L^2+3L_\noise^{1/2} +16 \bigr) \int\norm{\theta}^2\sdist(d\theta)
	                  +  16 L^4 + 2L_\noise + 128 L^6 + 8  L_\noise^{3/2} \Bigr]\,.
}}
	Additionally, if the test function is given as $\phi = \tilde \phi \circ f$ for a function $\tilde \phi$ that is 
	$L_{\tilde \phi}$-Lipschitz, it holds that  %
	\vspace{-.02in}
	\eq{
	\bigl|\sdist(\phi)-\phi(\theta^*)\bigr| \leq
	L_{\tilde \phi}\Biggl\{ g^{-1}\Bigl( \frac{2 M \eta}{2-\tilde L\eta} \Bigr) +\frac{2 M \eta}{2-\tilde L\eta} \Biggr\}.}

\end{theorem}

For smooth objectives with Lipschitz gradient, \cite{karimi2016linear} 
provide a linear rate under the PL-inequality 
(see also~\cite[Lemma 2]{de2017automated}), which yields the asymptotic bias
$|\sdist(\phi)-\phi(\theta^*)|\leq\bigo(\eta).$ The above result recovers their findings as a special case, 
and provides a considerable generalization.%

\subsubsection{Convexity} \label{sec:cvx}

To make the analysis of constant step size SGD complete, 
we digress from the main theme of this paper and consider the constant step size SGD 
in the non-strongly convex regime,
for which there is no bias characterization known to authors. 
We show that,
under the convexity assumption,
one can achieve the same bias control as in the case of PL-inequality.
\begin{theorem}\label{bias_wz_convex}
	Let the Assumptions~\ref{asm:growth},\ref{asm:dissipativity}, and~\ref{asm:noise_bias} hold for a convex function $f$. 
	Then, the SGD iterates with a step size satisfying
	$\eta <c_{L,\alpha}$ for $c_{L,\alpha}$ in~\eqref{eq:constants} 
	admit the stationary distribution $\sdist$, which satisfies
	\eq{
	\sdist (f)-f^* \! \le C\eta\,,
	}
	where
	\eq{
	C\coloneqq  {2\Bigl(3L^2+3L_\noise^{1/2}(1+(\beta/\alpha)^2)\Bigr)}\Bigl (\int\norm{\theta}^2\sdist(d\theta) +\norm{\theta^*}^2 \Bigr)+3L^2\norm{\theta^*}^2 +5L^2
	+{ 2L_\noise^{1/2}\Bigl(1+(\beta/\alpha)^2\Bigr)}\,.
	}
	Additionally, if the test function is given as $\phi = \tilde \phi \circ f$ for a function $\tilde \phi$ that is 
	$L_{\tilde \phi}$-Lipschitz, then,
	\eq{
	\bigl|\sdist(\phi)-\phi(\theta^*) \bigr| \leq L_{\tilde \phi}C\eta.
	}
\end{theorem}
Convexity implies that any critical point $\theta^*$ is a global minimizer, which is similar to the PL-inequality; yet,
it does not imply a unique minimizer unlike strong convexity. The resulting
step size dependency of the bias is the same as in the case of PL-inequality, which is because
both of these conditions assert a similar gradient-based domination criterion on the sub-optimality. That is, we have
in the convex case $\inprod{\nabla f(\theta)}{ \theta- \theta^*}\geq f(\theta) - f(\theta^*)$, and in the case of PL-inequality $\gamma^{-1}\|f(\theta)\|^2 \geq f(\theta) - f(\theta^*)$.

\section{Examples and Numerical Studies}

We now demonstrate the asymptotic normality and bias in non-convex optimization with two examples arising in robust statistics for which our assumptions can be verified. We consider the online SGD setting with the update rule~\eqref{ref:sgd_org_online} and also the semi-stochastic setting, where the noise sequence $\{\noise_k(\theta)\}_{k\ge 1}$ is independent of $\theta$ and is simply a sequence of i.i.d. random vectors -- such a setting helps to demonstrate how to verify our assumptions explicitly.

\subsection{Regularized MLE for heavy-tailed linear regression}\label{mle1}

While the least-squares loss function is common in the context of linear regression,
it is well-documented that it suffers from robustness issues when the error distribution
of the model is heavy-tailed~\cite{huber2004robust}. Indeed in fields like finance,
oftentimes the Student's $t$-distribution is used to model the heavy-tailed error~\cite{fan2017elements}. In this case, defining the random vector $Z:=(X,Y)$,  the stochastic optimization problem in~\eqref{eq:onlinesgd_main} is given by the expectation of the function
\eq{
F(Z,\theta):=\log\bigl(1+(Y-\inprod{X}{\theta})^2\bigr)+\frac{\lambda}{2}\norm{\theta}^2,
}
which is non-convex (as a function of $\theta$) for small penalty levels $\lambda$. Correspondingly, given $n$ independent and identically distributed samples $(\vecx_i, y_i)$, the finite-sum version of the optimization problem corresponds to minimizing the following objective function
\eq{
	\label{eq:reg-mle}
		f(\theta)\coloneqq \frac{1}{2m}\sum_{i=1}^m \log\bigl(1+(y_i-\inprod{\vecx_i}{\theta})^2\bigr)+\frac{\lambda}{2}\norm{\theta}^2.
}
We consider the finite-sum setup as we would be able to verify our assumptions and empirically demonstrate the bias result in a clean manner in this setup, as we demonstrate next. 

\subsubsection{Semi-stochastic Gradient Descent}\label{sec:ex1_fullsgd}
In the experiments, $\mathbf{X}:= (\vecx_1,\dots,\vecx_m)^\top\in\R^{m\times d}$ represents a fixed design matrix
generated from $\mathbf{X}_{ij}\sim \text{Bernoulli}(\pm 1)/\sqrt{d}$, 
and $\vecy:=(y_1,\dots,y_m)^\top\in\R^m$ represents the response vector generated according to
the linear model $y_i = \inprod{\vecx_i}{\theta_\text{true}} + \varepsilon$ with $(\theta_\text{true})_i \overset{\op{iid}}{\sim}\op{Unif}(0,1),$
and $\varepsilon$
is Student-t ($\text{df}=10$) noise. 
We choose $m=5000$, $d=10$, and the Lipshitz test function~$\phi(\theta)=\norm{\theta}$ unless stated otherwise.\vspace{0.1in}
\begin{figure}[t]
	\centering
	\hspace*{-.1in}	\includegraphics[width=6.9in]{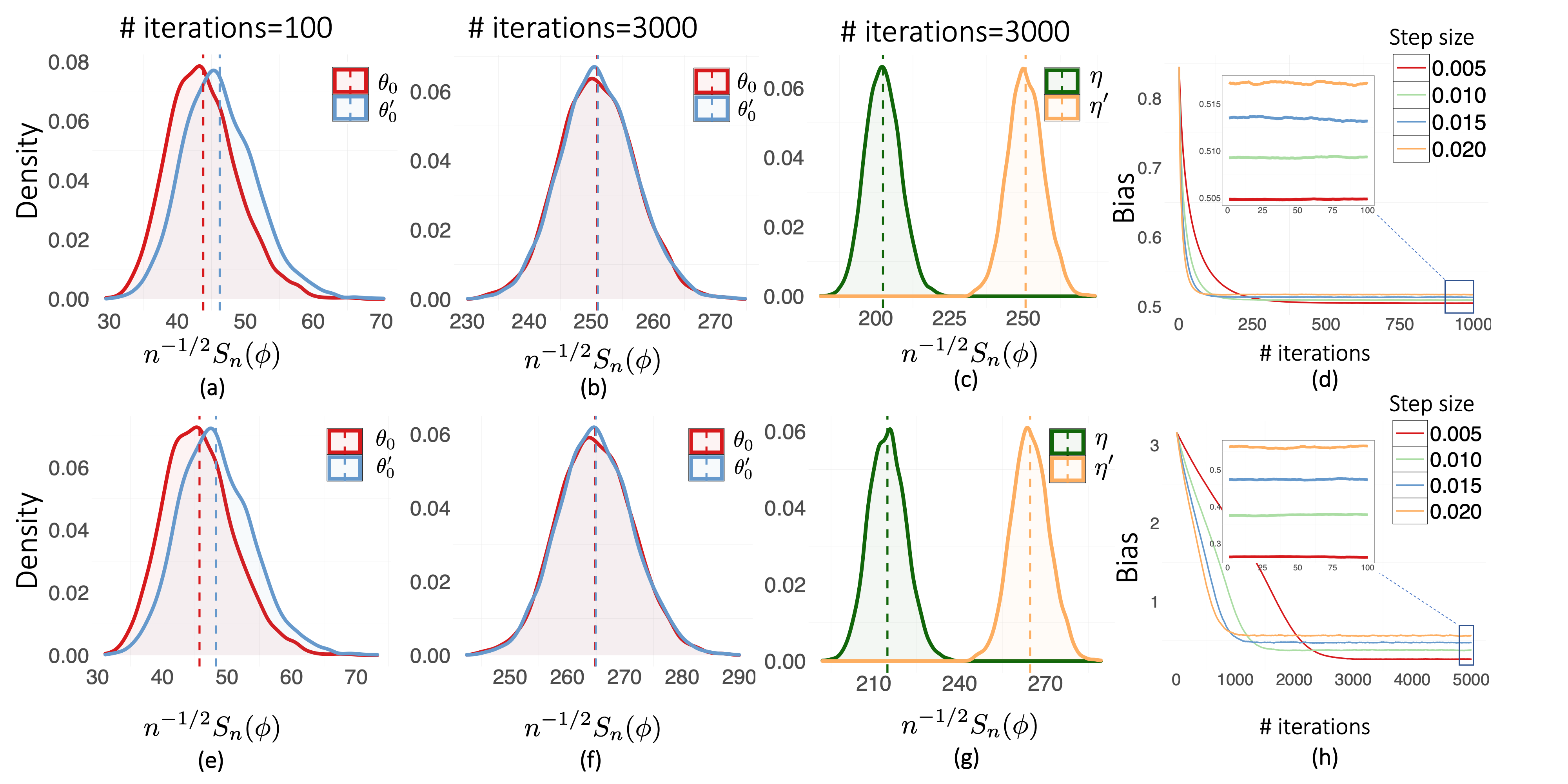}
	\caption{%
          First and second rows correspond to non-convex examples in Sections~\ref{sec:ex1_fullsgd} and~\ref{sec:ex2_fullsgd}, respectively.
          Figures (a,b), (e,f) show the density of $n^{-1/2}S_n(\phi)=n^{-1/2}\sum_{k=1}^n \phi(\iter_k)$ with different initializations (red, blue) for different number of iterations. Figures (c,g) show the same density with different step sizes. Figures (d,h) show the evolution of bias against the number of iterations.
	} 
	\label{fig:1}
\end{figure}
\begin{figure}[t]
	\centering
	\hspace*{-.1in}	\includegraphics[width=5.3in]{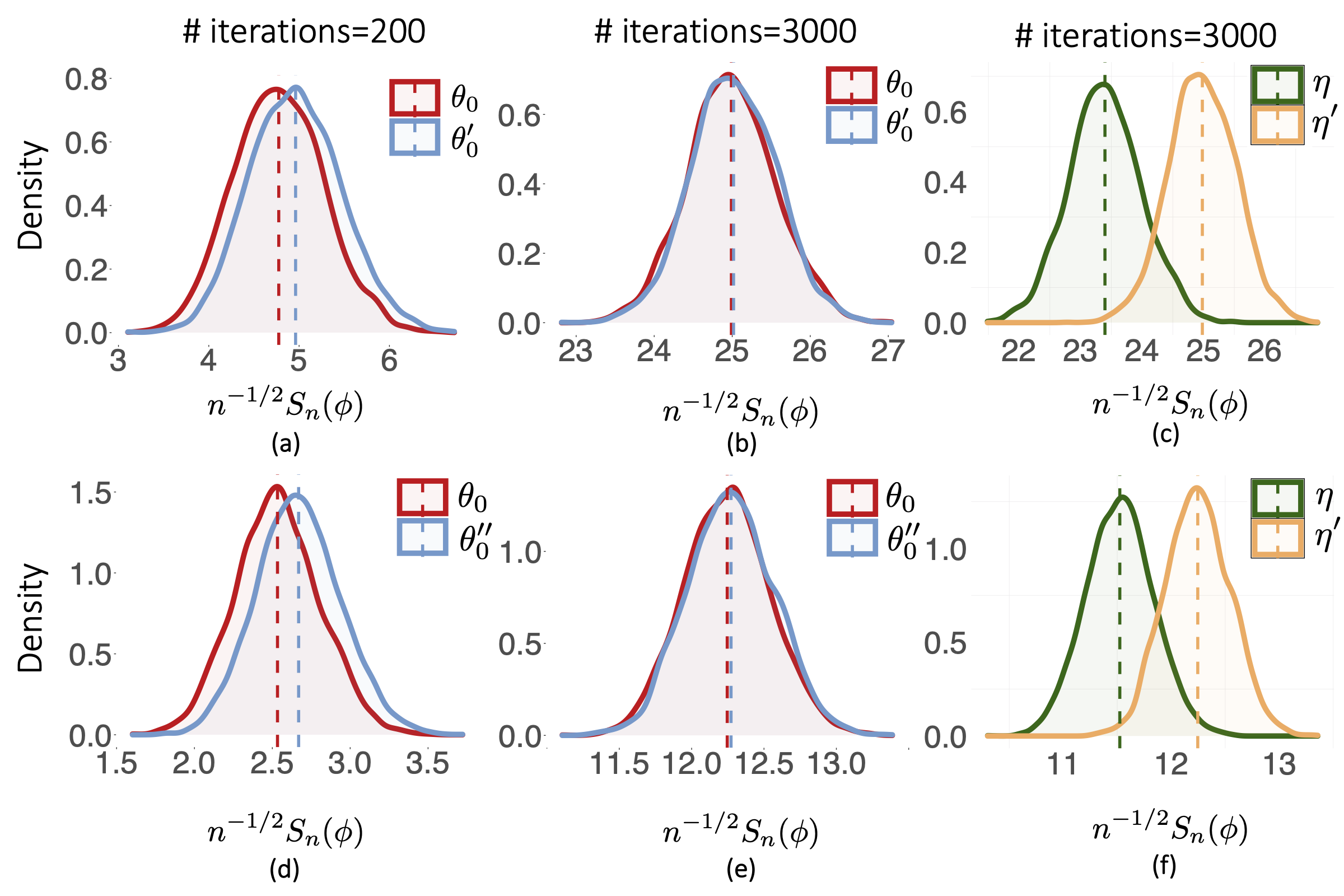}
	\caption{%
          First and second rows correspond to non-convex examples in Sections~\ref{sec:ex1_semisgd} and~\ref{sec:ex2_semisgd}, respectively.
          Figures (a,b), (e,f) show the density of $n^{-1/2}S_n(\phi)=n^{-1/2}\sum_{k=1}^n \phi(\iter_k)$ with different initializations (red, blue) for different number of iterations. 
          Figures (c,g) show the same density with different step sizes. 
	} 
	\label{fig:2}
\end{figure}

\noindent\textbf{Asymptotic normality:} Fig.~\ref{fig:1}-(a,b,c,d) demonstrates
the normality and the bias of SGD with heavy-tailed gradient noise distributed as Student-t ($\text{df}=5$).
Each plot has two density curves where red and blue curves in Fig.~\ref{fig:1}-(a,b) respectively correspond
to initializations with $\theta_0=(1,\dots,1)$
and $\theta'_0=(1.5,\dots,1.5)$ with step size $\eta = 0.3$;  green and orange curves in Fig.~\ref{fig:1}-c
correspond to step sizes $\eta=0.2$ and $\eta'=0.3$
with initialization~$\theta_0$. 
All experiments are based on 4000 Monte Carlo runs.
We observe in Fig.~\ref{fig:1}-a that different initializations have an early impact on the
normality when the number of iterations is moderate. However, when SGD is run for a longer time,
this effect is removed as in Fig.~\ref{fig:1}-b. Lastly, Fig.\ref{fig:1}-c demonstrates the
effect of step size on the normality,
where the means are different for different step sizes
as they depend on the stationary distribution $\pi_\eta$.
Indeed, the above results are not surprising.
One can verify that the objective function~\eqref{eq:reg-mle} satisfies Assumptions~\ref{asm:growth}
and \ref{asm:dissipativity}.
The above objective has the following gradient
\eqn{
	\nabla f(\theta) =\frac{1}{m}\sum_{i=1}^m\frac{\vecx_i(\inprod{\vecx_i}{\theta} - y_i)}{1+(y_i-\inprod{\vecx_i}{\theta})^2}
	+\lambda\theta.
}
Because
$
\norm{\nabla f(\theta)} \leq \left(\lambda_{\text{max}}(\tfrac{1}{m}\mathbf{X}^\top \mathbf{X}) + \lambda\right)\norm{\theta} +
\frac{1}{{m}}\norm{\mathbf{X}^\top\vecy}
$
by the triangle inequality and the fact that the denominator is lower bounded by 1, Assumption~\ref{asm:growth} holds.
For Assumption~\ref{asm:dissipativity}, we write
\eqn{
	\inprod{\nabla f(\theta)}{\theta} =\frac{1}{m}\sum_{i=1}^m\frac{(\inprod{\vecx_i}{\theta})^2-y_i\inprod{\vecx_i}{\theta}}{1+(y_i-\inprod{\vecx_i}{\theta})^2}
	+\lambda\norm{\theta}^2 \geq  - \big\|\frac{1}{m}\mathbf{X}^\top \vecy\big\| \|\theta\| + {\lambda}\|\theta\|^2,
}
by
Cauchy-Schwartz inequality.
Next, using Young's inequality
$
- \big\|\tfrac{1}{m}\mathbf{X}^\top \vecy\big\| \|\theta\| \!\geq \!-  \frac{1}{\lambda}\big\|\tfrac{1}{m}\mathbf{X}^\top \vecy\big\|^2 \!-\! \frac{\lambda}{4}\|\theta\|^2\,,
$
Assumption~\ref{asm:dissipativity} holds
for $\alpha\!=\!\lambda/4$ and 
$\beta \!=\! \frac{1}{\lambda}\big\|\tfrac{1}{m}\mathbf{X}^\top \vecy\big\|^2\!\!$.
Finally, the gradient noise has finite 4-th moment with support on $\mathbb{R}^d$; thus,
Assumption~\ref{asm:noise} is satisfied, and Theorem~\ref{general_clt_arb_sdist} is applicable.

\vspace{0.1in}
\noindent\textbf{Bias:} In order to demonstrate the bias behavior without speculation,
one needs the global minimum $\theta^*$ of the non-convex problem.
Therefore, we simplify the problem~\eqref{eq:reg-mle}
to another non-convex problem
\eq{\label{eq:reg-mle-simple}
	f(\theta)\coloneqq  \frac{1}{2} \log \bigl(1+\|\theta\|^2 \bigr)+\frac{\lambda}{2}\norm{\theta}^2\,.
}
Notice that the general structure is the same, with no data, and $\theta^*$ is known, i.e. $\theta^*=0$.

We choose the test function~$\phi(\theta)=\tilde\phi\circ f(\theta),$ where $\tilde\phi(x)=1/(1+e^{-x})$ is Lipschitz.
Fig.~\ref{fig:1}-(d) demonstrates how the bias $\pi_\eta(\phi) - \phi(\theta^*)$ changes over iterations, where
different curves correspond to different step sizes. We notice that larger step size provides fast initial
decrease; yet the resulting asymptotic bias is larger
which aligns with our theory. To verify assumptions, we compute the gradient and the Hessian respectively as

$$\nabla f(\theta) = \frac{\theta}{1+\|\theta\|^2 }+\lambda \theta, \qquad~\text{and}\qquad
\nabla^2 f(\theta) = \frac{I}{1+\|\theta\|^2} - \frac{2\theta\theta^\top}{(1+\|\theta\|^2)^2} + I\lambda,
$$
with $I$ denoting the identity matrix.
For small $\lambda$ the above function is clearly non-convex. To see this, choose $\lambda = 0.1, u = \theta/\|\theta\|$
and note that $\inprod{u}{\nabla^2 f(\theta)u} <0$ whenever $1.5 \leq \|\theta\|\leq 2$. 
Also, note that
$$
\|\nabla f(\theta)\|^2 = \|\theta\|^2 \Bigl(\lambda + 1/(1+\|\theta\|^2) \Bigr)^2 \geq \frac{2\lambda^2}{1+\lambda} \bigl\{f(\theta) - f(\theta^*)\bigr\}.
$$
Thus, Assumption~\ref{asm:loja} is satisfied for $\gamma=\frac{2\lambda^2}{1+\lambda}$ and $g(x)=\gamma x^2$.
Following the same steps in the regression setting, one can also verify Assumptions~\ref{asm:growth}-\ref{asm:noise}; hence, Theorem~\ref{bias_wz_loja} can be applied.

\subsubsection{Online Stochastic Gradient Descent}\label{sec:ex1_semisgd}
{
For our online SGD experiments, we use $b_k=2$, for all iterations $k$ to obtain the stochastic gradient. We also experimented with $m_k=10,50$ and observed similar behavior.  The distribution of the random vector $Z=(X,Y) \in \mathbb{R}^{d+1}$, is as follows: Each coordinate of the vector $X\in\rd$, is generated as $\text{Bernoulli}(\pm 1)/\sqrt{d}$ and given vector $X$, the response $Y\in\R$ is generated according to the linear model $Y= \inprod{X}{\theta_\text{true}} + \varepsilon$ with each coordinate of $\theta_\text{true}\in\rd$ generated from $\op{Unif}(0,1),$ and fixed, and $\varepsilon\in\R$ is Student-t ($\text{df}=10$) noise. 
We choose $d=10,$ and set a burn-in period of size 100.}\\

\noindent\textbf{Asymptotic normality:} Fig.~\ref{fig:2}-(a,b,c) demonstrates the normality of online SGD.
Each plot has two density curves where red and blue curves in Fig.~\ref{fig:2}-(a,b) respectively correspond
to initializations with $\theta_0=(1,\dots,1)$
and $\theta^{'}_0=(2.5,\dots,2.5)$ with step size $\eta = 0.3$;  
green and orange curves in Fig.~\ref{fig:2}-c
correspond to step sizes $\eta=0.2$ and $\eta'=0.3$
with initialization~$\theta_0$. 
All experiments are based on 4000 Monte Carlo runs.
We observe in Fig.~\ref{fig:2}-a that different initializations have an early impact on the
normality when the number of iterations is moderate. 
However, when SGD is run for a longer time,
this effect is removed as in Fig.~\ref{fig:2}-b. 
Lastly, Fig.\ref{fig:2}-c demonstrates the
effect of step size on the normality,
where the means are different for different step sizes
as they depend on the stationary distribution $\pi_\eta$.

\subsection{Regularized Blake-Zisserman MLE for corrupted linear regression}\label{mle2}

While the above example was based on linear-regression with heavy-tailed noise, we now consider the case of heavy-tailed regression with corrupted noise. In this setup, the noise model in linear regression is assumed to be Gaussian, but a fraction of the noise vectors are assumed to be corrupted in the sense that they are drawn from a uniform distribution. Such a scenario arises in visual reconstruction problems; see for example~\cite{blake1987visual} for details. In this case, defining the random vector $Z:=(X,Y)$,  the stochastic optimization problem in~\eqref{eq:onlinesgd_main} is given by the expectation of the function
\eq{
F(Z,\theta):= \log \Bigl(\nu + e^{-(Y-\inprod{X}{\theta})^2} \Bigr)+\frac{\lambda}{2}\norm{\theta}^2,~~~\nu>0.
}
Similar the previous case, we also consider the finite-sum version: Given $n$ independent and identically distributed samples $(\vecx_i, y_i)$, it corresponds to minimizing the following objective function
\eqn{
	f(\theta)=-\frac{1}{2m}\sum_{i=1}^m \log \Bigl(\nu + e^{-(y_i-\inprod{\vecx_i}{\theta})^2} \Bigr)+\frac{\lambda}{2}\norm{\theta}^2,~~~\nu>0\,.
}
\subsubsection{Semi-stochastic Gradient Descent}\label{sec:ex2_fullsgd}
In the experiments, we use the same setup and parameters as in Section~\ref{sec:ex1_fullsgd}.
\vspace{0.1in}

\noindent\textbf{Asymptotic normality:} Fig~\ref{fig:1}-(e,f,g) demonstrates
the asymptotic normality of the SGD with heavy-tailed gradient noise Student-t($\text{df}=6$).
The experimental setup is the same as the previous example with the same values for $\theta_0,\theta_0',\eta,\eta'$.
We observe the early impact of initialization in Fig~\ref{fig:1}-a,
the clear normality in Fig.~\ref{fig:1}-b, and the effect of step size on CLT in Fig.\ref{fig:1}-c.
These observations also align with our theory since this objective also satisfies our assumptions. Indeed, 
it has the gradient
\eqn{
	\grad f(\theta)&= -\frac{1}{m}\sum_{i=1}^m \frac{\vecx_i \bigl(y_i-\inprod{\vecx_i}{\theta} \bigr)e^{-(y_i-\inprod{\vecx_i}{\theta})^2}}{\nu+e^{-(y_i-\inprod{\vecx_i}{\theta})^2}}+\lambda\theta\,.
}
The triangle inequality yields 
$$
\norm{\grad f(\theta)}
\le \frac{1}{1+\nu}\big\|\tfrac{1}{m}\mathbf{X}^\top \vecy\big\|  + \Big( \frac{1}{1+\nu}\lambda_{\text{max}}(\tfrac{1}{m}\mathbf{X}^\top \mathbf{X})+\lambda\Big)\norm{\theta},
$$
which verifies Assumption~\ref{asm:growth}. 
To verify the dissipativity assumption, we can write
\eqn{
	\inprod{\grad f(\theta)}{\theta}
	= \inprod{-\frac{1}{m}\sum_{i=1}^m \frac{\vecx_i \bigl(y_i-\inprod{\vecx_i}{\theta}\bigr)e^{-(y_i-\inprod{\vecx_i}{\theta})^2}}{\nu+e^{-(y_i-\inprod{\vecx_i}{\theta})^2}}+\lambda\theta}{\theta}
	\ge  -\frac{1}{1+\nu}\big\|\tfrac{1}{m}\mathbf{X}^\top \vecy\big\| \norm{\theta}+\lambda\norm{\theta}^2\,.
}
The inequality follows from the triangle and Cauchy-Schwartz inequalities.
Using Young's inequality, we obtain
$$- \frac{1}{1+\nu}\big\|\tfrac{1}{m}\mathbf{X}^\top \vecy\big\| \|\theta\| \geq -  \frac{1}{\lambda(1+\nu)}\big\|\tfrac{1}{m}\mathbf{X}^\top \vecy\big\|^2 - \frac{\lambda}{4(1+\nu)}\|\theta\|^2,$$
which shows that the above function is dissipative for $\alpha=\lambda/2$ and 
$\beta = \frac{1}{2\lambda(1+\nu)^2}\big\|\tfrac{1}{m}\mathbf{X}^\top \vecy\big\|^2$; thus, Assumption~\ref{asm:dissipativity} holds.

\vspace{0.1in}
\noindent\textbf{Bias:} Similar to the previous example, we simplify the problem so that we can
compute the bias $\pi_\eta(\phi) - \phi(\theta^*)$. We consider the function
\eqn{
	f(\theta)\coloneqq -\frac{1}{2} \log\big(\nu + e^{-\|\theta\|^2}\big)+\frac{\lambda}{2}\norm{\theta}^2,~~~\nu>0\,.
}
We observe in Fig.\ref{fig:1}-h that smaller step sizes lead to smaller asymptotic bias. To verify that this can be predicted from our theory, we write the gradient and the Hessian respectively, as
$$\nabla f(\theta)=\frac{\theta}{1 + \nu e^{ \| \theta\|^2}} +\lambda \theta~\qquad\text{and}~\qquad
\nabla^2 f(\theta)=\frac{I}{1 + \nu e^{ \| \theta\|^2}}- \frac{2\nu e^{\|\theta\|^2}}{(1 + \nu e^{ \| \theta\|^2})^2}\theta\theta^\top +\lambda I.$$
First, note that the Hessian can have negative eigenvalues for small values of $\lambda$. For example, for $\nu=1$, $\lambda=0.1$, and the unit direction $u=\theta/\|\theta\|$, we have $\inprod{u}{\nabla^2 f(\theta)u} <0$
for $1 \leq \|\theta\|^2 \leq 2$; thus the function is non-convex.
But we also have
\eqn{
	\inprod{\nabla f(\theta)}{ \theta}= \|\theta\|^2 \Big(\lambda + 1/\big(1 + \nu e^{ \| \theta\|^2}\big)\Big) \geq \Big(\lambda + 1/\big(1 + \nu e^{ R^2}\big)\Big) \|\theta\|^2
}
for $\|\theta\| \leq R$ and $  \inprod{\nabla f(\theta)}{ \theta} \geq \lambda \|\theta\|^2$ for $\|\theta\|^2 >R$; thus, Assumption~\ref{asm:dissipativity2} is satisfied for $\alpha=\lambda$, and any
$\beta \ge 0$ and $g(x) = \Big(\lambda + 1/\big(1 + \nu e^{ R^2}\big)\Big) x^2$.
Following the same steps in the previous example, 
one can also verify Assumptions~\ref{asm:growth}-\ref{asm:noise}; therefore, Theorem~\ref{bias_wz_disp} follows.
\subsubsection{Online Stochastic Gradient Descent}\label{sec:ex2_semisgd}

In the experiments, we use the same setup as in Section~\ref{sec:ex1_semisgd}.
\vspace{0.1in}

\noindent\textbf{Asymptotic normality:} Fig.~\ref{fig:2}-(d,e,f) demonstrates the normality of online SGD.
Each plot has two density curves where red and blue curves in Fig.~\ref{fig:2}-(d,e) respectively correspond
to initializations with $\theta_0=(1,\dots,1)$
and $\theta^{''}_0=(1.5,\dots,1.5)$ with step size $\eta = 0.3$;  
green and orange curves in Fig.~\ref{fig:2}-c
correspond to step sizes $\eta=0.2$ and $\eta'=0.3$
with initialization~$\theta_0$. 
All experiments are based on 4000 Monte Carlo runs.
We observe in Fig.~\ref{fig:2}-d that different initializations have an early impact on the
normality when the number of iterations are moderate. 
However, when SGD is run for a longer time,
this effect is removed as in Fig.~\ref{fig:2}-e. 
Lastly, Fig.\ref{fig:2}-f demonstrates the
effect of step size on the normality,
where the means are different for different step sizes
as they depend on the stationary distribution $\pi_\eta$.

\section{Discussions}\label{sec:discussion}

By leveraging the connection between constant step size SGD and Markov chains~\cite{dieuleveut2017}, we provided theoretical results characterizing the bias and the fluctuations of constant step size SGD for non-convex and non-smooth optimization which arises frequently in modern statistical learning. 

\vspace{0.1in}
\noindent\textbf{Estimating the Asymptotic Variance:} As discussed in Section~\ref{sec:cltheory}, in order for using the established CLT to compute confidence intervals in practice, the population expectation $\sdist(\phi)$ and asymptotic variance $\sigma^2_{\sdist}(\phi)$ have to be estimated. We suggest the following three ways to do so:

\begin{itemize}%
	\item Estimate them based on sample average of a single trajectory of SGD iterates, i.e., the mean $\sdist(\phi)$ is estimated as $n^{-1}\sum_{k=0}^{n-1}\phi\big(\iter_{k}\big)$ and the variance $\sigma^2_{\sdist}(\phi)$ by adopting the online approach of~\cite{zhu2020fully} to the constant step size setting.%
	\item First run $N$ parallel SGD trajectories and compute the average of each trajectory, to obtain $N$ independent observations from the stationary distribution $\pi_\eta$. Next, use the $N$ observations to compute the sample mean and the sample variance estimators for $\sdist(\phi)$ and $\sigma^2_{\sdist}(\phi)$. %
	\item Leverage the online bootstrap and variance estimation approaches proposed in~\cite{fang2018online, su2018statistical,chen2020statistical} for the constant step size SGD setting in order to obtain estimates for $\sdist(\phi)$ and $\sigma^2_{\sdist}(\phi)$.  
\end{itemize}
A theoretical investigation on the relative merits of the above approaches is left as future work.

\section*{Acknowledgements}
MAE is partially funded by NSERC Grant [2019-06167], Connaught New Researcher Award, CIFAR AI Chairs program, and CIFAR AI Catalyst grant. KB is partially supported by a seed grant from the Center for Data Science and Artificial Intelligence Research, UC Davis. SV is partially supported by a discovery grant from NSERC of Canada and a Connaught New Researcher Award. 
{The authors thank Yichen Zhang for helpful comments on an earlier version of this manuscript.}

\bibliographystyle{amsalpha}
\bibliography{./bib}

\newpage

\appendix

\section{Proofs for Sections~\ref{sec:cltheory} and~\ref{sec:bias}}\label{sec:cltproofs}

\subsection{Preliminaries and Additional  Notations}

Note that the sequence of iterates $\{\iter_k\}_{k\ge 0}$ is a homogeneous Markov chain~\cite{dieuleveut2017}. 
We denote the (sub-)$\sigma$-algebra (of~$\mathcal{F}$) of events up to and including the $k$-th iteration as $\filtration_k.$
By definition, the discrete-time stochastic process defined in~\eqref{ref:sgd_org} is adapted to its natural filtration~$\{\filtration_k\}_{k\ge 0}$. 
We denote the Markov kernel on $(\rd,\mathcal{B}(\rd))$ associated with SGD iterates~\eqref{ref:sgd_org} by $\kernel$ with 
\begin{equation*}
\kernel(\iter_k,A)=\mpr(\iter_{k+1}\in A| \iter_{k}) ~~\mpr - a.s.,~~~\forall A\in\borel(\rd), k\ge 0\,.
\end{equation*}
Define the $k$-th power of this kernel iteratively: define $\kernel^1:=\kernel,$ and for $k\ge 1,$ for all $\tilde\theta\in\rd$ and $A\in\mathcal{B}(\rd),$ define
\begin{equation*}
\kernel^{k+1}(\tilde\theta,A):=\int_{\rd}\kernel(\tilde\theta,d\theta)\kernel^k (\theta,A)\,.
\end{equation*} 
For any function~$\phi:\rd\to \R$ and $k\ge 0,$ define the measurable function~$\kernel^k\phi(\theta):\rd\to \R$ for all $\theta\in\rd$ via
 \begin{equation*}
\kernel^k\phi(\theta)=\int \phi(\tilde\theta)\kernel^k(\theta,d\tilde\theta)\,.
\end{equation*}
Given the $L_\phi$-Lipschitz function~$\phi:\rd\to\R$ and the expectation of $\phi$ under the stationary measure $\pi_\eta,$ define the function $h$ as
\begin{align*}
h: \rd &\to \R \\
 \theta&\mapsto \phi(\theta)- \sdist(\phi)\,.
\end{align*}  
Note that $\sdist(h)=0$ and $h$ is $L_\phi$-Lipschitz. 
Define the partial sum~$S_n(\phi):=\sum_{k=0}^{n-1}h(\iter_{k}).$ Moreover, we define 
\begin{equation*}
\bariter :=\int _{\rd} \theta d\sdist (\theta)\,.
\end{equation*}  
\subsection{Proofs of Proposition~\ref{ergodicity} and Theorem~\ref{general_clt_arb_sdist}}
We start with some preliminary results required to prove the CLT.
\begin{lemma}\label{lyapunov}
Under Assumptions~\ref{asm:growth}-\ref{asm:noise}, it holds for any 
$\eta\in \left(0,\frac{\alpha-\sqrt{(\alpha^2-{(3L^2+L_\noise)})\vee 0 }}{{3L^2+L_\noise}}\right)$ and any fixed initial point~$\iter_0=\theta_0\in\rd$ that
\begin{equation*}
\E[\,\norm{\iter_{k+1}}^2+1|\filtration_k]\le \alpha_{\dagger}(\,\norm{\iter_k}^2+1)+\beta_{\dagger}\,.
\end{equation*}
Here, $\alpha_{\dagger}\in(0,1)$ and $\beta_{\dagger}\in(0,\infty)$ are constants depending on $\eta.$ 
The explicit formulas of $\alpha_{\dagger},\beta_{\dagger}$ are given in the proof.
\end{lemma}

\begin{proof}[Proof of Lemma~\ref{lyapunov}]
Define $U_{\eta}:=\frac{\alpha-\sqrt{\max \{\alpha^2-{(3L^2+L_\noise)},0 \}}}{  {3L^2+L_\noise} }.$
Given $\eta \in(0,U_{\eta}),$ define 
\begin{equation*}
\alpha_{\dagger}=1+\eta^2{(3L^2+L_\noise)}-2\eta\alpha
\end{equation*}
{and note that with this definition $\alpha_{\dagger} \in (0,1)$ whenever $\eta \in(0,U_{\eta})$.}
Then, with $\eta, \alpha_{\dagger}$, and the fixed initial point~$\iter_0=\theta_0\in\rd,$ we set 
\begin{equation*}
\beta_{\dagger}:={\kappa(\alpha_{\dagger}^{1/2}-\alpha_{\dagger}) }\,,
\end{equation*}
where
\begin{align*}
\kappa:=& \frac{4\eta(\alpha+\beta)+12\eta^2L^2}{ {\alpha_{\dagger}^{1/2}-\alpha_{\dagger}}}
\bigvee 1\,.
\end{align*}
It follows that $\beta_{\dagger}>0.$
Note that
\begin{align*}
&\E[1+\norm{\iter_{k+1}}^2|\filtration_{k}]\\
=&\E[1+\norm{\iter_{k} -\eta\bigl(\grad f(\iter_{k}) + {\noise_{k+1}(\iter_k)}  \bigr)}^2|\filtration_{k}]\\
=&1+\E\bigl[ \,\norm{\iter_{k}}^2 +\eta^2 \,\norm{\grad f(\iter_{k})}^2  +  \eta^2 \,\norm{ {\noise_{k+1}(\iter_k)}  }^2-2\eta\inprod{\iter_k}{\grad f(\iter_k)}|\filtration_k\bigr]\,.
\end{align*}
The last step follows from the Assumption~\ref{asm:noise}.
By Assumption~\ref{asm:growth}, we have
\begin{equation*}
\norm{\grad f(\iter_{k})}^2 \le L^2(1+\norm{\iter_{k}})^2\,.
\end{equation*}
Squaring both sides and using the fact that $(1+\norm{\iter_{k}})^2\le 3(\,\norm{\iter_{k}}^2+3)$ gives
\begin{equation*}
\norm{\grad f(\iter_{k})}^2 \le 3L^2(\,\norm{\iter_{k}}^2+3)\,.
\end{equation*}
By Assumption~\ref{asm:dissipativity}, we obtain
\begin{equation*}
\inprod{\iter_k}{\grad f(\iter_k)}\ge \alpha\,\norm{\iter_k}^2-\beta\,.
\end{equation*}
By Assumption~\ref{asm:noise}, it holds that  
\begin{equation*}
{\E[\,\norm{ \noise_{k+1}(\iter_k) }^2 |\filtration_k]\le L_{\noise}(1+\norm{\iter_k}^2 )\,. }
\end{equation*}
Plugging the previous three inequalities into the first display provides us with
\begin{equation}\label{re:2mom}
\E[1+\norm{\iter_{k+1}}^2|\filtration_{k}]
\le 1+9\eta^2L^2+{\eta^2L_\noise}+2\eta\beta+(1+3\eta^2L^2+{\eta^2L_\noise}-2\eta\alpha)\,\norm{\iter_{k}}^2\,.
\end{equation}
Recall that $\alpha_{\dagger}=1+\eta^2{(3L^2+L_\noise)}-2\eta\alpha.$
Plugging $\alpha_{\dagger}$ back into the previous display yields 
\begin{equation}\label{ref:drift1}
\E[\,\norm{\iter_{k+1}}^2+1|\filtration_k]\le \alpha_{\dagger}(\,\norm{\iter_k}^2+1)+2\eta(\alpha+\beta)+6\eta^2L^2\,.
\end{equation}
Note that $\beta_{\dagger}={\kappa(\alpha_{\dagger}^{1/2}-\alpha_{\dagger}) },$ where 
\begin{equation*}
\kappa\ge  \frac{4\eta(\alpha+\beta)+12\eta^2L^2}{ {\alpha_{\dagger}^{1/2}-\alpha_{\dagger} } }\,.
\end{equation*}
It then follows that $\E[\,\norm{\iter_{k+1}}^2+1|\filtration_k]\le \alpha_{\dagger}(\,\norm{\iter_k}^2+1)+\beta_{\dagger}$ as desired.
\end{proof}

\begin{corollary}[Bounded second moment]\label{bouned2moment}
Under the assumptions stated in Lemma~\ref{lyapunov}, with the constant step size 
$\eta\in \left(0,\frac{\alpha-\sqrt{(\alpha^2-{(3L^2+L_\noise)})\vee 0 }}{{3L^2+L_\noise}}\right)$
the stationary distribution $\sdist$ satisfies 
\begin{equation*}
\mu_{2,\eta} := \int\norm{\theta}^2\sdist(d\theta) \le {3+ \frac{2\beta}{\alpha}}\,.
\end{equation*}
\end{corollary}
\begin{proof}[Proof of Corollary~\ref{bouned2moment}]
Consider the chain~$\chain$ starting from the stationary distribution~$\sdist.$
	By display~\eqref{re:2mom}, it holds that 
	\begin{equation*}
	\E[\norm{\iter_{k+1}}^2]
	\le 9\eta^2L^2+{\eta^2L_\noise}+2\eta\beta+(1+3\eta^2L^2+{\eta^2L_\noise}-2\eta\alpha)\,\norm{\iter_{k}}^2\,.
	\end{equation*}
	Using the fact that by stationarity $\E[\norm{\iter_{k+1}}^2] = \E[\norm{\iter_{k}}^2]$ and rearranging the previous display gives
	\begin{equation*}
	\E[\norm{\iter_{k}}^2]
	\le \frac{9\eta L^2+{\eta L_\noise}+2\beta}{2\alpha-{\eta (3L^2+L_\noise)} }
	\le {3+ \frac{2\beta}{\alpha}}\,.
	\end{equation*}
\end{proof}

\begin{corollary}[Lyapunov condition]\label{drift}
Under the assumptions stated in Lemma~\ref{lyapunov}, {given the step size specified in Lemma~\ref{lyapunov}}, it holds that
\begin{equation*}
\E[V(\iter_{k+1})|\filtration_{k}]\le \alpha_{\dagger} V(\iter_k) +\beta_{\dagger}\,,
\end{equation*}
where the  Lyapunov function~$V(\theta)$ is defined via
\begin{equation}\label{ref:lyapunov_func}
V(\theta):=\norm{\theta}^2+1\,.
\end{equation}
{Observe that by the proof of Lemma {15.2.8} in~\cite{meyn2012} this also implies that the drift condition (V4) in~\cite{meyn2012} holds with $V$ defined above, $b = \beta_\dagger, \beta = (1-\alpha_\dagger)/2$ and the following set~$\mathcal{C}$
\begin{equation}\label{ref:smallset}
\mathcal{C}:= \Big\{\theta\in\rd:V(\theta)\le \frac{2\beta_{\dagger}}{\gamma-\alpha_{\dagger}} \Big\}\,,
\end{equation}
for an arbitrary but fixed $\gamma\in(\alpha_{\dagger}^{1/2},1).$ }
\end{corollary}

\begin{corollary}[Minorization condition]\label{remark:minorization}
Under  Assumptions~\ref{asm:growth}-\ref{asm:noise}, {given the step size specified in Lemma~\ref{lyapunov}},
there exists a constant~$\zeta>0$, and a probability measure~$\nu^{\dagger}$ (depending on $\eta$ which is suppressed in the notation) with $\nu^{\dagger}(\mathcal{C})=1$ and $\nu^{\dagger}(\mathcal{C}^c)=0,$ such that
\begin{equation*}
\kernel(\theta,A)\ge \zeta \nu^{\dagger}(A)
\end{equation*}
holds for any $A\in\mathcal{B}(\rd)$ and $\theta\in \mathcal{C}$ for the set $\mathcal{C}$ defined in~\eqref{ref:smallset}.
\end{corollary}

\begin{proof}[Proof of Corollary~\ref{remark:minorization}]
Recall the definition of the markov chain~\eqref{ref:sgd_org}, we have
\begin{equation*}
{\noise_{k+1}(\iter_k)}= \frac{\iter_k-\iter_{k+1}}{\eta}-\grad f(\iter_k)\,.
\end{equation*}
Recall that the distribution of {$\noise_1(\theta)$} can be decomposed as $\mu_{1,\theta} + \mu_{2,\theta}$ where $\mu_{1,\theta}$ has density $p_{\theta}$.
It then holds for any $\theta\in\rd$ that 
\begin{equation}\label{minor}
\kernel(\theta,\mathcal{C})=\mpr(\iter_{k+1}\in\mathcal{C}|\iter_k=\theta) \geq \int_{t\in\mathcal{C}} \frac{1}{\eta^d}\, {p_{\theta}}\Bigl(\frac{\theta-t}{\eta}-\grad f(\theta) \Bigr) dt>0\,.
\end{equation}
This implies every state in the state space is within reach of any other state over the set $\mathcal{C}.$
Define the probability measure~$\nu^{\dagger}$ with density
\[
p_{\nu^{\dagger}}(t) := I\{\theta \in \mathcal{C}\} \frac{\inf_{\theta\in\mathcal{C}} p(t|\theta) }{\int_{t\in\mathcal{C}} \inf_{\theta\in\mathcal{C}} p(t|\theta) dt}\,,
\]
and set the constant $\zeta:= \int_{t\in\mathcal{C}} \inf_{\theta\in\mathcal{C}} p(t|\theta) dt.$
By Assumption~\ref{asm:noise} and the display~\eqref{minor}, it holds that $\zeta>0,$  $\nu^{\dagger}(\mathcal{C})=1$ and $\nu^{\dagger}(\mathcal{C}^c)=0.$ 
Moreover, it holds that any $A\in\mathcal{B}(\rd)$ and $\theta\in \mathcal{C}$ that 
\begin{equation*}
\kernel(\theta,A)\ge \zeta \nu^{\dagger}(A)\,.
\end{equation*}
This implies the minorization condition is met {for all choices of $\eta$ given by Lemma~\ref{lyapunov}}.
\end{proof}

\begin{lemma}\label{aperiodic_harris}
Under Assumptions~\ref{asm:growth}-\ref{asm:noise}, the chain~$\chain$ is an aperiodic, $\psi$-irreducible, and Harris recurrent chain, with an invariant measure~$\sdist.$
\end{lemma}
\begin{remark}
This lemma implies the chain~$\chain$  is positive.
\end{remark}
\begin{proof}[Proof of Lemma~\ref{aperiodic_harris}]
\textbf{Step 1}: We show that the chain~$\chain$ is aperiodic.
By Assumption~\ref{asm:noise}, there does not exist $d\ge 2$ and a partition of size~$d+1$ such that $\borel(\rd) = (\dot\cup_{i=1}^d D_i) \dot\cup N,$ where $\dot\cup$ denotes the disjoint union, and $N$ is a $\psi$-null (transient) set, such that $\kernel(\theta, D_{i+1})=1$ holds for $\psi$-a.e. $\theta\in D_i.$
Thus, the largest \textit{period} of the chain defined in~\eqref{ref:sgd_org} is 1, which implies the chain is aperiodic.\\
\textbf{Step 2}: We show that the chain~$\chain$ is $\psi$-irreducible, and recurrent with an invariant probability measure.
We note that by Assumption~\ref{asm:noise}, there exists some non-zero $\sigma$-finite measure~$\psi$ on~$(\rd,\borel(\rd))$ such that for any $\theta\in\rd$ and any $A\in\borel(\rd)$ with $\psi(A)>0,$ it holds that 
\begin{equation*}
\mpr(\iter_{k+1}\in A|\iter_k=\theta) 
\geq \int_{\tilde\theta\in A} \frac{1}{\eta^d}\, {p_{\theta}}\Bigl(\frac{\theta-\tilde\theta}{\eta}-\grad f(\theta) \Bigr) d\tilde\theta >0\,,
\end{equation*}
where $p_{\theta}$ was defined in Assumption~\ref{asm:noise}. This implies the Markov chain defined in~\eqref{ref:sgd_org} is $\psi$-irreducible.
By the Lyapunov condition established in Corollary~\ref{drift}, part (iii) of Theorem 15.0.1 in~\cite{meyn2012} holds.
It then follows by condition (i) of this theorem that  the chain~$\chain$ is recurrent with an  invariant probability measure $\sdist.$\\
\textbf{Step 3}: We show that the chain is Harris recurrent.
Define the hitting time~$\tau_{\mathcal{C}}:=\inf\{n>0: \iter_n\in \mathcal{C}\},$ where the set~$\mathcal{C}$ is defined in~\eqref{ref:smallset}. 
By Corollary A.4 in \cite{mattingly2002}, it holds for any fixed $\iter_0=\theta_0\in\rd$ that
\begin{align*}
\mpr(\tau_{\mathcal{C}}<\infty) = 1\,.
\end{align*} 
By Proposition 10.2.4 in~\cite{douc2018}, the chain is Harris recurrent.
\end{proof}

\noindent Now, we are ready to prove Proposition~\ref{ergodicity}.
\vspace{-0.1in}
\begin{proof}[Proof of Proposition~\ref{ergodicity}]
(a) By Lemma~\ref{aperiodic_harris}, the chain~$\chain$ is an aperiodic Harris recurrent chain, with an invariant measure~$\sdist.$
Note that the chain is also positive. Thus condition (i) of Theorem 13.0.1 in~\cite{meyn2012} is satisfied and this implies the existence of a unique invariant measure~$\sdist$.  
The fact that this stationary distribution has a finite second moment was established in Corollary~\ref{bouned2moment}.
\\
(b)
By Lemma~\ref{aperiodic_harris}, the iterates~$\chain$ are realiztions from a $\psi$-irreducible and aperiodic chain. Note that 
\begin{align*}
|\phi(\theta)|\le& \kappa_{\phi}(1+\norm{\theta})\\
\le&2\kappa_{\phi}\sqrt{1+\norm{\theta}^2}\\
\le & 2\kappa_{\phi}V(\theta)\,.
\end{align*}
By Corollary~\ref{drift}, the condition (iv) of Theorem 16.0.1 in~\cite{meyn2012} with $V(\theta)=2\kappa_{\phi}(1+\norm{\theta}^2)$ is fulfilled.
By part (ii) in that theorem, it holds that for fixed $\iter_0=\theta_0\in\rd$
\begin{align*}
|\kernel^k\phi(\theta_0)-\sdist(\phi)|
 \le & \kappa \rho^{k} V(\theta_0)\,,
\end{align*}
where $\rho\in(0,1), \kappa > 0$ are constants depending on $\phi$.
\end{proof}

\noindent We now prove Theorem~\ref{general_clt_arb_sdist}. In order to do so, we first derive the central limit theorem for the function $h$ when the Markov chain starting from its stationary distribution $\sdist.$ 
\begin{lemma}[CLT with stationary initial distribution]
\label{general_clt_sdist}
Assume Assumptions~\ref{asm:growth}-\ref{asm:noise} hold.
For any step size~$\eta\in \left(0,\frac{\alpha-\sqrt{(\alpha^2-{(3L^2+L_\noise)})\vee 0 }}{{3L^2+L_\noise}}\right)$, it holds that 
\begin{equation*}
n^{-1/2}S_n(\phi) \underset{\mpr_{\sdist}}{\longrightarrow}\mathcal{N}(0,\sigma^2_{\sdist}(\phi))\,,
\end{equation*}  
where $\sigma^2_{\sdist}(\phi)= 2\sdist(h\hat h)-\sdist(h^2)$ with $\hat h = \sum_{k=0}^\infty \kernel^k h.$
\end{lemma}

\begin{proof}[Proof of Lemma~\ref{general_clt_sdist}]
We prove the claim by appealing to Theorem 17.0.1 in \cite{meyn2012}. 
In order to do so, we first show that the chain~$\chain$ is $V$-uniformly ergodic, 
where the function $V$ is defined in~\eqref{ref:lyapunov_func}. 
Then, we establish the CLT by employing Theorem 17.0.1 in \cite{meyn2012}. 
\\
\textbf{Step 1}: We show that the chain~$\chain$ is $V$-uniformly ergodic.
By Lemma~\ref{aperiodic_harris} and Proposition~\ref{ergodicity}, the chain~$\chain$ is positive Harris recurrent with a unique stationary distribution~$\sdist.$
Note that the chain~$\chain$ is also $\psi$-irreducible and aperiodic. 
By Corollary~\ref{drift}, condition (iv) of Theorem 16.0.1 in~\cite{meyn2012} is satisfied.
Then, it follows from part (i) of this theorem that the iterates~$\chain$ is $V$-uniformly ergodic.\\
\textbf{Step 2}: We now establish the CLT for the averaged SGD iterates starting from the stationary distribution~$\sdist$. 
Note that for the test function~$\phi(\theta),$ it holds for any $\theta\in\rd$ that
\begin{equation*}
|\phi(\theta)|\le \kappa_{\phi}(1+\norm{\theta})\le 2\kappa_{\phi}\sqrt{1+\norm{\theta}^2}\,,
\end{equation*}
which implies
\begin{equation*}
|\phi(\theta)|^2 \le 4\kappa^2_{\phi}V(\theta)\,.
\end{equation*}
Thus the conditions required to leverage Theorem 17.0.1 (ii), (iv) with $g(\theta) = \phi(\theta)$ in~\cite{meyn2012} are satisfied. Hence, by Theorem 17.0.1 in~\cite{meyn2012}, we obtain
\begin{equation*}
\frac{1}{\sqrt{n}}\sum_{k=0}^{n-1} h(\iter_k) \underset{\mpr_{\sdist}}{\longrightarrow}\mathcal{N}(0,\sigma^2_{\sdist}(\phi))\,,
\end{equation*}  
where $\sigma^2_{\sdist}(\phi)=2\sdist(h\hat h)-\sdist(h^2)>0.$
\end{proof}

\begin{proof}[Proof of Theorem~\ref{general_clt_arb_sdist}]
By Lemma~\ref{aperiodic_harris} and Lemma~\ref{general_clt_sdist}, the desired result follows readily from Proposition 17.1.6 in~\cite{meyn2012}.
\end{proof}

\subsection{Proofs of Proposition~\ref{bias_wo_local}, Theorems~\ref{bias_wz_disp},~\ref{bias_wz_loja}, and~\ref{bias_wz_convex}}
We need the following auxiliary lemma.

\begin{lemma}\label{dissipativity_global}
Assumptions~\ref{asm:growth} and~\ref{asm:dissipativity} implies
\begin{align*}
\inprod{\grad f(\theta)}{\theta-\theta^*}\ge 
\alpha'\norm{\theta-\theta^*}^2-\beta'  \,,
\end{align*}
where $\theta^*\in\rd$ is any critical point of function~$f$, and $\alpha',\beta'$ are positive constants.
\end{lemma}
\begin{proof}[Proof of Lemma~\ref{dissipativity_global}]
When $\theta^*=\mathbf{0},$ the result follows trivially from Assumption~\ref{asm:dissipativity}.
Assume $\norm{\theta^*}>0.$
Note that 
\begin{align*}
\inprod{\grad f(\theta)}{\theta-\theta^*} =   \inprod{\grad f(\theta)}{\theta}- \inprod{\grad f(\theta)}{\theta^*}\,.
\end{align*}
By Assumption~~\ref{asm:dissipativity}, it holds that
\begin{align*}
 \inprod{\grad f(\theta)}{\theta}&\ge \alpha\norm{\theta}^2-\beta\\
 &\ge  \alpha (\norm{\theta-\theta^*}^2+\norm{\theta^*}^2-2\norm{\theta^*}\norm{\theta-\theta^*})-\beta\,.
\end{align*}
By Assumption~~\ref{asm:growth}, Cauchy-Schwarz inequality and triangular inequality, it holds that
\begin{align*}
\inprod{\grad f(\theta)}{\theta^*}&\le \norm{\grad f(\theta)}\norm{\theta^*} \le  L\norm{\theta^*}(1+\norm{\theta-\theta^*}+\norm{\theta^*})\,.
\end{align*}
Combing the previous two displays yields 
\begin{align*}
&\inprod{\grad f(\theta)}{\theta-\theta^*} \\
\ge& \alpha (\norm{\theta-\theta^*}^2+\norm{\theta^*}^2-2\norm{\theta^*}\norm{\theta-\theta^*})-\beta -L\norm{\theta^*}(1+\norm{\theta-\theta^*}+\norm{\theta^*})  \\
\ge & \frac{\alpha}{2}\norm{\theta-\theta^*}^2-\beta-L\norm{\theta^*}^2-L\norm{\theta^*}\,.
\end{align*}
The desired result follows by setting $\alpha':=\frac{\alpha}{2}$ and $\beta':=\beta+\bigl(   \sqrt{\alpha } +\frac{2L}{\sqrt{\alpha}}\bigr)^2\norm{\theta^*}^2+L\norm{\theta^*}.$
\end{proof}

{
\begin{lemma}\label{noise_local}
Under Assumptions~\ref{asm:dissipativity} and~\ref{asm:noise_bias}, it holds for any $k\ge 1$ and $\theta\in\rd$ that
\begin{align*}
\E[\norm{\noise_{k+1}(\theta)}^{r}]\le {L'_{\noise}}^{r/4}(1+\norm{\theta-\theta^*}^r)\,, \ \text{ for }\ \ r \in\{2,3,4\},
\end{align*}
 where $\theta^*\in\rd$ is any critical point of function~$f$, and $L'_\noise :=8L_\noise(1+(\beta/\alpha)^4).$
 \end{lemma}
\begin{proof}[Proof of Lemma~\ref{noise_local}]
By Assumptions~\ref{asm:dissipativity} and~\ref{asm:noise_bias}, it holds that
\begin{align*}
 \E[\norm{\noise_{k+1}(\theta)}^4]
 &\le L_{\noise}(1+\norm{\theta}^4)\\
 &\le L_{\noise}(1+8\norm{\theta-\theta^*}^4+8\norm{\theta^*}^4)\\
 &\le L_{\noise}(1+8\norm{\theta-\theta^*}^4+8(\beta/\alpha)^4)\\
 &\le L'_\noise (1+\norm{\theta-\theta^*}^4 )\,,
\end{align*}
where $L'_\noise :=8L_\noise(1+(\beta/\alpha)^4).$
Similarly, for $r\in \{2,3\}$ we have
\begin{align*}
 \E[\norm{\noise_{k+1}(\theta)}^r]
 &\leq  \E[\norm{\noise_{k+1}(\theta)}^4]^{r/4}\\
 &\le L_{\noise}^{r/4}(1+\norm{\theta}^4)^{r/4}\\
  &\le L_{\noise}^{r/4}(1+\norm{\theta}^r)\\
 &\le L_{\noise}^{r/4}(1+2^{r-1}\norm{\theta-\theta^*}^r+2^{r-1}\norm{\theta^*}^r)\\
 &\le L_{\noise}^{r/4}(1+2^{r-1}\norm{\theta-\theta^*}^r+2^{r-1}(\beta/\alpha)^r)\\
 &\le {L'_\noise}^{r/4} (1+\norm{\theta-\theta^*}^r )\,,
\end{align*}
where ${L'_\noise}$ is defined above.
\end{proof}
}

\begin{lemma}\label{4mom_decomp}
Under Assumptions~\ref{asm:growth},~\ref{asm:dissipativity}, and~\ref{asm:noise_bias}, {with step size~$\eta< 1\wedge \frac{1}{10 \bar L}$,} it holds for any $k\ge 0$ that
\begin{align}
&\E[\norm{\iter_{k+1}-\theta^*}^4|\filtration_k]\\
 \le&  (1-4\eta\alpha'+32L_\dagger\eta^2) \norm{\iter_k-\theta^*}^4
     + \eta (4\beta'+24\bar L^2+12{ {L'_\noise}^{1/2}} +64)\norm{\iter_k-\theta^*}^2\\
     &+ \eta^2(64\bar L^4 + 8{L'_\noise}   + 32( 4\bar L^3)^2 + 32  {(L'_\noise)^{3/2}}) \label{eq:bound4mom}\,.
\end{align}
where {$L_\dagger:= \bar L^2+ 16\Bigl( L_\noise^{3/4}(1+(\beta/\alpha)^3 )\vee 
L_\noise^{1/2}(1+(\beta/\alpha)^2 )\vee
L_\noise(1+(\beta/\alpha)^4 )
\Bigr) $ 
with $\bar L:=L(1+\norm{\theta^*})$, $L'_\noise$ is from  Lemma~\ref{noise_local}, and $\theta^*$ is any critical points of fuction~$f.$}
\end{lemma}
\begin{proof}[Proof of Lemma~\ref{4mom_decomp}]
Define $\Delta_k:=\norm{\iter_k-\theta^*}.$
It holds by Assumption~\ref{asm:growth} that
\begin{equation*}
\norm{\grad f(\iter_k)}\le \bar L\Delta_k+\bar L\,,
\end{equation*}
where $\bar L=L(\norm{\theta^*}+1).$
Note that
\begin{align}
\label{ref:4mom_decp}
\begin{split}
\Delta_{k+1}^4
=& (\Delta_k^2+\eta^2\,\norm{\grad f(\iter_k)+{\noise_{k+1}(\iter_k)} }^2-2\eta\inprod{\grad f(\iter_k)+{\noise_{k+1}(\iter_k)}}{\iter_k-\theta^*})^2\\
=&\Delta_k^4+\eta^4\,\norm{\grad f(\iter_k)+{\noise_{k+1}(\iter_k)}}^4+4\eta^2 \inprod{\grad f(\iter_k)+{\noise_{k+1}(\iter_k)}}{\iter_k-\theta^*}^2\\
&+2\eta^2\Delta_k^2\,\norm{\grad f(\iter_k)+{\noise_{k+1}(\iter_k)}}^2-4\eta\Delta_k^2\inprod{\grad f(\iter_k)+{\noise_{k+1}(\iter_k)} }{\iter_k-\theta^*}\\
&-4\eta^3 \,\norm{\grad f(\iter_k)+{\noise_{k+1}(\iter_k)} }^2\inprod{\grad f(\iter_k)+{\noise_{k+1}(\iter_k)} }{\iter_k-\theta^*}\\
=&\Delta_k^4+ \RN{1}+\RN{2}+\RN{3}+\RN{4}+\RN{5}\,,
\end{split}
\end{align}
where
\begin{align*}
\RN{1}&:=\eta^4\,\norm{\grad f(\iter_k)+{\noise_{k+1}(\iter_k)}}^4\\
\RN{2}&:=4\eta^2 \inprod{\grad f(\iter_k)+{\noise_{k+1}(\iter_k)}}{\iter_k-\theta^*}^2\\
\RN{3}&:=2\eta^2\Delta_k^2\,\norm{\grad f(\iter_k)+{\noise_{k+1}(\iter_k)}}^2\\
\RN{4}&:=-4\eta\Delta_k^2\inprod{\grad f(\iter_k)+{\noise_{k+1}(\iter_k)}}{\iter_k-\theta^*}\\
\RN{5}&:=-4\eta^3 \,\norm{\grad f(\iter_k)+{\noise_{k+1}(\iter_k)}}^2\inprod{\grad f(\iter_k)+{\noise_{k+1}(\iter_k)}}{\iter_k-\theta^*}\,.
\end{align*}
To obtain the expectation~$\E[\Delta_{k+1}^4],$ we first calculate the conditional expectation~$\E[\Delta_{k+1}^4|\filtration_k]$.
For this, we proceed the conditional expectation of the above five terms separately.
Note that 
\begin{align*}
\E[\RN{1}|\filtration_k]=&\eta^4\E[\,\norm{\grad f(\iter_k)+{\noise_{k+1}(\iter_k)}}^4|\filtration_k]\\
\le &\eta^4\E[8\,\norm{\grad f(\iter_k)}^4 +8\,\norm{{\noise_{k+1}(\iter_k)}}^4|\filtration_k]\\
\le& 8\eta^4 (8\bar L^4\Delta_k^4+ 8\bar L^4 +{ L'_\noise \Delta_k^4 +L'_\noise} )\,.
\end{align*}
The first inequality follows from the fact that $(x+y)^4\le 8(x^4+y^4), \forall x, y >0.$
The last inequality follows from Assumptions~\ref{asm:growth} and Lemma~\ref{noise_local}.
Using the same trick and invoking Cauchy-Schwarz inequality gives 
\begin{align*}
\E[\RN{2}|\filtration_k]=&4\eta^2\E[ \inprod{\grad f(\iter_k)+{\noise_{k+1}(\iter_k)} }{\iter_k-\theta^*}^2|\filtration_k]\\
\le & 4\eta^2\Delta_k^2 \E[\,\norm{\grad f(\iter_k)+{\noise_{k+1}(\iter_k)} }^2|\filtration_k] \\
\le & 8\eta^2\Delta_k^2 (2\bar L^2 \Delta_k^2+2\bar L^2+{ {L'_\noise}^{1/2} \Delta_k^2 + {L'_\noise}^{1/2}}  )\,.
\end{align*}
Similarly, we have
\begin{align*}
\E[\RN{3}|\filtration_k]=& 2\eta^2\Delta_k^2\E[ \,\norm{\grad f(\iter_k)+{\noise_{k+1}(\iter_k)} }^2|\filtration_k]\\
\le &  4\eta^2\Delta_k^2(2\bar L^2 \Delta_k^2+2\bar L^2+{ {L'_\noise}^{1/2} \Delta_k^2 + {L'_\noise}^{1/2}} )\,.
\end{align*}
Using Cauchy-Schwarz inequality again, we obtain
\begin{align*}
\E[\RN{5}|\filtration_k]
=& \E[ -4\eta^3 \,\norm{\grad f(\iter_k)+{\noise_{k+1}(\iter_k)}}^2\inprod{\grad f(\iter_k)+{\noise_{k+1}(\iter_k)}}{\iter_k-\theta^*}|\filtration_k]\\
\le & 4\eta^3 \E[\norm{\grad f(\iter_k)+{\noise_{k+1}(\iter_k)}}^3\norm{\iter_k-\theta^*}|\filtration_k] \\
= & 4\eta^3 \Delta_k\E\bigl[\norm{\grad f(\iter_k)+{\noise_{k+1}(\iter_k)}}^3|\filtration_k\bigr] \\
\le & 4\eta^3 \Delta_k\E\bigl[4\norm{\grad f(\iter_k)}^3+4\norm{{\noise_{k+1}(\iter_k)}}^3|\filtration_k\bigr]\,.
\end{align*}
{Note that by Lemma~\ref{noise_local}, it holds for any $k\ge 1$ and $\theta\in\rd$ that 
\begin{align*}
\E[\norm{\noise_k(\theta)}^3] \le {L'_{\noise}}^{3/4} (1+\norm{\theta-\theta^*}^3)\,.
\end{align*}
Combining this with the previous display yields
\begin{align*}
\E[\RN{5}|\filtration_k]
\le & 16\eta^3\Delta_k(4\bar L^3\Delta_k^3+ 4\bar L^3+{L'_{\noise}}^{3/4}+{L'_{\noise}}^{3/4}\Delta_k^3) \\
=& 64\bar L^3\eta^3\Delta_k^4+ 16\eta^3{L'_{\noise}}^{3/4}\Delta_k^4 + 16\eta^2(\Delta_k\eta 4\bar L^3+\Delta_k\eta{L'_{\noise}}^{3/4})\,.
\end{align*}
}
Collecting pieces gives 
\begin{align}
\begin{split}
\E[\Delta_{k+1}^4|\filtration_k] 
\le& \Delta_{k}^4(1+64\eta^4\bar L^4+64\eta^3\bar L^3+24\eta^2\bar L^2 +{8\eta^2L'_\noise+ 12 \eta^2{L'_\noise}^{1/2} +16\eta^2 {L'_{\noise}}^{3/4}})\\
        &-4\eta\Delta_k^2\inprod{\grad f(\iter_k)}{\iter_k-\theta^*}\\
       &+\eta^2\bigl(64\eta^2\bar L^4+ 8\eta^2{L'_\noise}+ 24\bar L^2 \Delta_{k}^2+12{L'_\noise}\Delta_{k}^2
       + 64 \Delta_{k}^2 + 32(\eta 4\bar L^3)^2 + 32 (\eta{{L'_\noise}^{3/4}})^2\bigr)\\
\le & \Delta_{k}^4[1 +32\eta^2(\bar L^2+ {L'_{\noise}+{L'_\noise}^{1/2} +{L'_{\noise}}^{3/4}  })]-4\eta\Delta_k^2\inprod{\grad f(\iter_k)}{\iter_k-\theta^*}\\
    &+\eta\bigl(64\bar L^4 \eta+ 8{L'_\noise} \eta + 24\bar L^2 \Delta_{k}^2+12{ {L'_\noise}^{1/2}}\Delta_{k}^2
    + 64 \Delta_{k}^2 + 32( 4\bar L^3)^2 \eta+ 32 {(L'_\noise)^{3/2}}\eta\bigr)\,.
\end{split}
\end{align}
The above inequalities are based on the fact that $\eta<\frac{1}{10\bar L} \wedge 1$ and $xy\le 2x^2+ 2y^2,\forall x, y >0.$
By Lemma~\ref{dissipativity_global}, we handle the term~$\RN{4}$ as following  
\begin{align*}
\E[\Delta_{k+1}^4|\filtration_k]
\le &\Delta_{k}^4 \bigl(1-4\eta\alpha'+32\eta^2(\bar L^2+  {L'_{\noise}+{L'_\noise}^{1/2} +{L'_{\noise}}^{3/4}  }) \bigr) \\
&+ \eta\bigl(4\beta'\Delta_{k}^2+ 64\bar L^4 \eta+ 8{L'_\noise} \eta + 24\bar L^2 \Delta_{k}^2+12{{L'_\noise}^{1/2}}\Delta_{k}^2
    + 64 \Delta_{k}^2 + 32( 4\bar L^3)^2 \eta+ 32 {{L'_\noise}^{3/2}}\eta\bigr)\,.
\end{align*}
{Define $L_\dagger:= \bar L^2+ 16\Bigl( L_\noise^{3/4}(1+(\beta/\alpha)^3 )\vee 
L_\noise^{1/2}(1+(\beta/\alpha)^2 )\vee
L_\noise(1+(\beta/\alpha)^4 )
\Bigr) $.
Note that $L_\dagger>\bar L^2+  L'_{\noise}+{L'_{\noise}}^{1/2}+{L'_{\noise}}^{3/4}.$ 
Combing this with the previous display gives
\begin{align}
&\E[\Delta_{k+1}^4|\filtration_k]\\
\le &\Delta_{k}^4 (1-4\eta\alpha'+32\eta^2L_\dagger ) \\
&+ \eta\bigl(4\beta'\Delta_{k}^2+ 64\bar L^4 \eta+ 8{L'_\noise} \eta + 24\bar L^2 \Delta_{k}^2+12{ {L'_\noise}^{1/2}}\Delta_{k}^2
    + 64 \Delta_{k}^2 + 32( 4\bar L^3)^2 \eta+ 32 {{L'_\noise}^{3/2}}\eta\bigr) \\
 \le&  (1-4\eta\alpha'+32L_\dagger\eta^2) \Delta_{k}^4 
     + \eta (4\beta'+24\bar L^2+12{ {L'_\noise}^{1/2}} +64) \Delta_{k}^2
     + \eta^2(64\bar L^4 + 8{L'_\noise}   + 32( 4\bar L^3)^2 + 32  {{L'_\noise}^{3/2}})   \,. 
\end{align}}

\end{proof}

\begin{lemma}\label{bound4mom}
Assume Assumptions~\ref{asm:growth}-\ref{asm:noise} holds. With the step size~
\begin{equation*}
\eta \le  \frac{\alpha-\sqrt{(\alpha^2-{(3L^2+L_\noise)})\vee 0 }}{{3L^2+L_\noise}} 
\wedge  \frac{\alpha}{64 L_\dagger} 
\wedge 1\,,
\end{equation*}
 the chain~\eqref{ref:sgd_org} has the stationary distribution~$\sdist$, and the chain has finite 4-th moment:
\begin{align*}
&\E[\norm{\iter_{k+1}}^4] \le \mu_{4,\eta}\,,
\end{align*}
where
\begin{align*}
\mu_{4,\eta}:=\frac{8}{7\alpha}\Bigl( (\beta+6 L^2+3  {L_\noise^{1/2}}  +16) \mu_{2,\eta} +   16 L^4 + 2  {L_\noise} + 128 L^6 + 8  {L_\noise^{3/2}} \Bigr)\,
\end{align*}
with $\mu_{2,\eta}$ defined in Corollary~\ref{bouned2moment},
and  { $L_\dagger$ defined in Lemma~\ref{4mom_decomp}. }
\end{lemma}
\begin{proof}[Proof of Lemma~\ref{bound4mom}]
Similar to display~\eqref{eq:bound4mom}, we can derive 
\begin{align*}
\E[\norm{\iter_{k+1}}^4|\filtration_k]
\le & (1-4\eta\alpha+32  {L_0^\dagger}\eta^2) \norm{\iter_{k}}^4 \\
&+ \eta\bigl[(4\beta+24 L^2+12 {L_\noise^{1/2}} +64)\norm{\iter_{k}}^2+ \eta(64 L^4 + 8 {L_\noise}   + 32( 4 L^3)^2 + 32 {L_\noise^{3/2}} )\bigr]\,,
\end{align*}
{where $L_0^\dagger:= L^2+  L_{\noise}+L_{\noise}^{1/2}+L_{\noise}^{3/4}$.
Recall the definition of $L_\dagger$ in Lemma~\ref{4mom_decomp}, it holds that $L_\dagger\ge L_0^\dagger,$ which implies  
\begin{align*}
\E[\norm{\iter_{k+1}}^4|\filtration_k]
\le & (1-4\eta\alpha+32 {L_\dagger}\eta^2) \norm{\iter_{k}}^4 \\
&+ \eta\bigl[(4\beta+24 L^2+12{L_\noise^{1/2}} +64)\norm{\iter_{k}}^2+ \eta(64 L^4 + 8 {L_\noise}   + 32( 4 L^3)^2 + 32 {L_\noise^{3/2}} )\bigr]\,,
\end{align*}
 }
Note that the chain starts from the stationary distribution~$\sdist$, taking the expectation on both sides gives 
\begin{align*}
&(4\eta\alpha-32 {L_\dagger}\eta^2) \E[\norm{\iter_{k}}^4]\\
\le &  \eta(4\beta+24 L^2+12 {L_\noise^{1/2}} +64)\E[\norm{\iter_{k}}^2]+ \eta^2 (64 L^4 + 8 {L_\noise}   + 32( 4 L^3)^2 + 32 {L_\noise^{3/2}}  )\,.
\end{align*}
We also note that $\E[\norm{\iter_{k}}^2] = \mu_{2,\eta}$ for $\mu_{2,\eta}$ from Corollary~\ref{bouned2moment}. 
Plugging this into the previous display and rearranging the inequality gives
\begin{align*}
&\E[\norm{\iter_{k+1}}^4]\\
\le & \frac{\eta}{4\eta\alpha-32 {L_\dagger} \eta^2}(4\beta+24 L^2+12 {L_\noise^{1/2}} +64) \mu_{2,\eta}+ \frac{\eta^2}{4\eta\alpha-32 {L_\dagger} \eta^2} (64 L^4 + 8 {L_\noise} 
 + 32( 4 L^3)^2 + 32  {L_\noise^{3/2}} )\\
 \le &  \frac{\eta}{4\eta\alpha-32 {L_\dagger}\eta^2}(4\beta+24 L^2+12 {L_\noise^{1/2}} +64) \mu_{2,\eta}+ \frac{\eta}{4\eta\alpha-32 {L_\dagger}\eta^2} (64 L^4 + 8 {L_\noise}
 + 32( 4 L^3)^2 + 32 {L_\noise^{3/2}} )\\
 \le & \frac{2}{7\alpha}\Bigl[(4\beta+24 L^2+12 {L_\noise^{1/2}}  +64) \mu_{2,\eta} +   (64 L^4 + 8 {L_\noise} + 32( 4 L^3)^2 + 32 {L_\noise^{3/2}} ) \Bigr]
\end{align*}
as desired.
\end{proof}

We are now ready to prove Proposition~\ref{bias_wo_local}.
\begin{proof}[Proof of Proposition~\ref{bias_wo_local}]
Define $\Delta_k:=\norm{\iter_k-\theta^*}.$
By Lemma~\ref{4mom_decomp}, we have
\begin{align}
&\E[\Delta_{k+1}^4|\filtration_k]\\
 \le&  (1-4\eta\alpha'+32L_\dagger\eta^2) \Delta_{k}^4
     + \eta (4\beta'+24\bar L^2+12{ {L'_\noise}^{1/2}} +64)\Delta_{k}^2\\
     &+ \eta^2(64\bar L^4 + 8{L'_\noise}   + 32( 4\bar L^3)^2 + 32  {{L'_\noise}^{3/2}}) \,.
\end{align}
Taking expectation on both sides then gives
\begin{align} 
&\E[\Delta_{k+1}^4]\\
\le & (1-4\eta\alpha'+32 {L_\dagger}\eta^2) \E[\Delta_{k}^4]
      + \eta(4\beta'+24\bar L^2+12{ {L'_\noise}^{1/2}} +64)\E[\Delta_{k}^2]\\
      &+ \eta^2(64\bar L^4 + 8{L'_\noise}  + 32( 4\bar L^3)^2 + 32 {{L'_\noise}^{3/2}})\,.
\end{align}
Set 
\begin{align*}
\varrho&:= 1-4\eta\alpha'+32{L_\dagger}\eta^2\\
A_1&:= 64\bar L^4+8{L'_\noise}+32(4\bar L^3)^2 + 32{{L'_\noise}^{3/2}} \\
A_2&:= 4\beta'+24\bar L^2+12{ {L'_\noise}^{1/2}}+64\,.
\end{align*}
By Cauchy-Schwatz inequality, we then have
\begin{align*}
\E[\Delta_{k+1}^4]
\le \varrho\E[\Delta_{k}^4]+ A_1\eta^2 + A_2 \E^{1/2}[\Delta_{k}^4]\eta\,.
\end{align*}
Note that when   $0<\eta< \frac{\alpha'-\sqrt{(\alpha'^2-4{L_\dagger})}}{16{L_\dagger}}\1(\alpha'^2> 8 {L_\dagger})
+\frac{\alpha'}{32{L_\dagger}}\1(\alpha'^2\le 8 {L_\dagger}),$ it follows that 
\begin{equation*}
\varrho > \frac{1}{2}\1 (\alpha'^2\ge 8{L_\dagger}) + (1-\frac{3\alpha'^2}{32{L_\dagger^2}})\1 (\alpha'^2< 8{L_\dagger})\ge\frac{1}{4}\,.
\end{equation*}
Set $D:=\sqrt{A_1}\vee A_2.$
We then find 
\begin{align*}
\E^{1/2}[\Delta_{k+1}^4] \le \sqrt{\varrho}\, \E^{1/2}[\Delta_{k}^4] + D\eta\,.
\end{align*}
By a straightforward induction, we have
\begin{align*}
\E^{1/2}[\Delta_{k}^4] \le \varrho^{k/2}\E^{1/2}[\Delta_{0}^4] + \frac{D\eta}{1-\sqrt{\varrho}}\,.
\end{align*}
Notice that $\eta\le \frac{\alpha'}{16{L_\dagger}},$ it then follows that
\begin{align*}
\varrho &= 1-4\eta\alpha'+32{L_\dagger}\eta^2 \le 1-2\eta\alpha'\,,
\end{align*}
which implies
\begin{align*}
\frac{1}{1-\sqrt{\varrho}}\le  \frac{1}{1-\sqrt{1-2\eta\alpha'}} \le  \frac{1}{\eta\alpha'}\,.
\end{align*}
Combining this with previous display gives 
\begin{align*}
\E^{1/2}[\Delta_{k}^4] \le \varrho^{k/2}\E^{1/2}[\Delta_{0}^4] + \frac{D}{\alpha'}\,.
\end{align*}
By Proposition~\ref{ergodicity}, there exists a unique stationary distribution~$\sdist$.

Consider the chain starting from the stationary distribution~$\sdist.$
Note that $\E[\Delta_{0}^4]\le 8(\E[\norm{\iter_0}^4]+\norm{\theta^*}^4).$
By Lemma~\ref{bound4mom}, it follows that 
\begin{equation*}
\E[\Delta_{0}^4]\le 8\mu_{4,\eta} + 8\norm{\theta^*}^4\,,
\end{equation*}
where the constant~$\mu_{4,\eta}$ is defined in Lemma~\ref{bound4mom}.
Plugging this into previous display provides us with

\begin{equation*}
\Bigl(\int\norm{\theta-\theta^*}^4\sdist(d\theta)\Bigr)^{1/4}=\mathcal{O}(1)\,.
\end{equation*}
Note that it holds for the $L_\phi$-Lipschitz continuous test function $\phi$ that
\begin{align*}
|\sdist(\phi)-\phi(\theta^*)|&\le L_{\phi}\int\norm{\theta-\theta^*}\sdist(d\theta) \\
&\le L_{\phi}\Bigl[\int\norm{\theta-\theta^*}^4\sdist(d\theta)\Bigr]^{1/4}\,,
\end{align*}
Thus, we obtain
\begin{align*}
|\sdist(\phi)-\phi(\theta^*)|=\bigo(1)
\end{align*}
as desired.

\end{proof}

\begin{lemma}\label{truncation}
For any $a,b,\delta>0,$ it holds for any $x\ge \frac{\delta}{a}+\sqrt{\frac{b}{a}}$ that 
\begin{equation*}
ax^2-b\ge \delta x\,.
\end{equation*}
\end{lemma}
\begin{proof}[Proof of Lemma~\ref{truncation}]
Define the function~$h(x):= ax^2-b-\delta x.$
When $x\ge \frac{\delta+\sqrt{\delta^2+4ab}}{2a},$ it holds that $h(x)\ge 0.$
Note that $\sqrt{\delta^2+4ab}\le \delta+\sqrt{4ab},$ it follows that when 
\begin{equation*}
x\ge \frac{\delta +\delta+\sqrt{4ab}}{2a}\,,
\end{equation*}
it holds that $h(x)\ge 0$.
The desired result then follows readily.
\end{proof}

\begin{proof}[Proof of Theorem~\ref{bias_wz_disp}]
  Consider the chain~$\chain$ starting from the stationary distribution~$\sdist.$
  Define~$\Delta_{k}:= \norm{\iter_{k}-\theta^*}.$
  Note that under Assumptions~\ref{asm:dissipativity2} and~\ref{asm:noise_bias}, Lemma~\ref{noise_local} still holds.
  By Assumptions~\ref{asm:growth},~\ref{asm:noise_bias}, and Lemma~\ref{noise_local} , we have
\begin{align*}
&\E\bigl[\Delta_{k+1}^2|\filtration_{k}\bigr]\\
=&\E\bigl[\Delta_{k}^2 +\eta^2 \,\norm{\grad f(\iter_{k})}^2  +  \eta^2 \,\norm{{\noise_{k+1}(\iter_k)} }^2-2\eta\inprod{\grad f(\iter_k)}{\iter_k-\theta^*}|\filtration_k\bigr]\\
\le &   \Delta_{k}^2 +\eta^2\bigl(3L^2(2\Delta_k^2 +2\norm{\theta^*}^2+3) + {{L'_\noise}^{1/2}(1+\Delta_{k}^2)} \bigr)-2\eta\inprod{\grad f(\iter_k)}{\iter_k-\theta^*}\\
=&  \Delta_{k}^2 +6L^2\eta^2\Delta_k^2 +{{L'_\noise}^{1/2} \eta^2 \Delta_k^2  }+\eta^2C_1-2\eta\inprod{\grad f(\iter_k)}{\iter_k-\theta^*}  
\end{align*}
where $C_1:= 6\norm{\theta^*}^2L^2+9L^2+{{L'_\noise}^{1/2}}.$
Note that the chain starts from the stationary distribution~$\sdist,$ which implies $\E[\Delta_{k+1}^2]=\E[\Delta_k^2]$ for all $k\ge 0.$
Taking the expectation on both sides and rearranging the inequality yields
\begin{equation}
\E[\inprod{\grad f(\iter_k)}{\iter_k-\theta^*}]\le {\eta (3L^2+{L'_\noise}^{1/2})}\E[\Delta_{k}^2]+\frac{\eta}{2}C_1\,.
\end{equation}
By Corollary~\ref{bouned2moment}, it follows that
\begin{equation}\label{bias_inq}
\E[\inprod{\grad f(\iter_k)}{\iter_k-\theta^*}]\le  C_2\eta\,,
\end{equation}
where $C_2 := { 2(3L^2+{L'_\noise}^{1/2})(\mu_{2,\eta} +\norm{\theta^*}^2)}+ C_1/2$ and $\mu_{2,\eta}$ is defined in Corollary~\ref{bouned2moment}.
Moreover, by Assumption~\ref{asm:dissipativity2}, Lemma~\ref{truncation}, and Jensen's inequality, we have
  \begin{equation*}
    \E[\inprod{\grad f(\iter_k)}{\iter_k-\theta^*}]\ge \delta\E[\Delta_{k}\1 (\Delta_{k}\ge R)]+ g(\E[\Delta_{k}\1 (\Delta_{k}< R)])\,.
  \end{equation*}
  Combining this with previous display provides us with
  \begin{align*}
    \E[\Delta_{k}\1 (\Delta_{k}\ge R)]\le \frac{C_2}{\delta}\eta \,,
  \end{align*}
  and 
  \begin{align*}
    \E[\Delta_{k}\1 (\Delta_{k}< R)]\le g^{-1}(C_2\eta)\,.
  \end{align*}
  Collecting pieces then gives
  \begin{align*}
    \E\bigl[\Delta_{k}\bigr]=&\E\bigl[\Delta_{k}\1 (\norm{\iter_k-\theta^*}<R)\bigr]+\E\bigl[\Delta_{k}\1 (\norm{\iter_k-\theta^*}\ge R)\bigr]\\
    \le &  \frac{C_2}{\delta}\eta+ g^{-1}( C_2\eta)\,.
  \end{align*}
  Thus, it holds for the $L_\phi$-Lipschitz continuous test function $\phi$ that
  \begin{equation*}
    |\sdist(\phi)-\phi(\theta^*)|\le L_{\phi}\int\norm{\theta-\theta^*}\sdist(d\theta)
    \leq L_{\phi} \Big( \frac{C_2}{\delta}\eta+ g^{-1}( C_2\eta) \Big)\,.
  \end{equation*}
\end{proof}

\begin{proof}[Proof of Theorem~\ref{bias_wz_loja}]

Consider the chain~$\chain$ starting from the stationary distribution~$\sdist.$
Note that by the assumption that~$\norm{\grad^2 f(\theta)}\le \tilde L(1+\norm{\theta})$ and Taylor expansion, we have
\begin{align*}
f(\iter_{k+1})
=& f(\iter_k)+\inprod{\grad f(\iter_k)}{\iter_{k+1}-\iter_k}+\frac{1}{2} (\iter_{k+1}-\iter_k)^\top\grad^2 f(\tilde\theta) (\iter_{k+1}-\iter_k)\\
\le &  f(\iter_k)+\inprod{\grad f(\iter_k)}{\iter_{k+1}-\iter_k} +\frac{1}{2}\tilde L\norm{ \iter_{k+1}-\iter_k}^2 (1+\norm{\tilde\theta})\,,
\end{align*}
where $\tilde\theta\in\rd$ is a convex combination between $\iter_{k+1}$ and $\iter_{k}.$
By definition of SGD iterates in~$\eqref{ref:sgd_org},$ it follows that 
\begin{align*}
f(\iter_{k+1})
\le &  f(\iter_{k})-\eta\inprod{\grad f(\iter_{k})}{\grad f(\iter_{k})+{\noise_{k+1}(\iter_k)} }+\frac{\tilde L}{2}\eta^2\norm{\grad f(\iter_{k})+{\noise_{k+1}(\iter_k)} }^2(1+\norm{\tilde\theta})\\
=&  f(\iter_{k})-\eta\inprod{\grad f(\iter_{k})}{\grad f(\iter_{k})+{\noise_{k+1}(\iter_k)} }\\
    &+\frac{\tilde L}{2}\eta^2 \bigl(\norm{\grad f(\iter_{k}}^2+\norm{{\noise_{k+1}(\iter_k)} }^2+2\inprod{\grad f(\iter_{k})}{{\noise_{k+1}(\iter_k)}}\bigr)(1+\norm{\tilde\theta})\\
\le &  f(\iter_{k})-\eta\inprod{\grad f(\iter_{k})}{\grad f(\iter_{k})+{\noise_{k+1}(\iter_k)}}\\
    &+\frac{\tilde L}{2}\eta^2\bigl(\norm{\grad f(\iter_{k}}^2+\norm{{\noise_{k+1}(\iter_k)}}^2+2\inprod{\grad f(\iter_{k})}{{\noise_{k+1}(\iter_k)}}\bigr)(1+\max\{\norm{\iter_k},\norm{\iter_{k+1}}\})\\
 \le&  f(\iter_{k})-\eta\norm{\grad f(\iter_{k})}^2-\eta \inprod{\grad f(\iter_{k})}{{\noise_{k+1}(\iter_k)}}\\
    &+\frac{\tilde L}{2}\eta^2\bigl(\norm{\grad f(\iter_{k}}^2+\norm{{\noise_{k+1}(\iter_k)}}^2
    +2\inprod{\grad f(\iter_{k})}{{\noise_{k+1}(\iter_k)}}\bigr)\\
    &+\frac{\tilde L}{2}\eta^2\bigl(\norm{\grad f(\iter_{k}}^2+\norm{{\noise_{k+1}(\iter_k)}}^2
    +2\inprod{\grad f(\iter_{k})}{{\noise_{k+1}(\iter_k)}}\bigr)(\norm{\iter_k}+\norm{\iter_{k+1}})\,.
\end{align*}
Taking the conditional expectation on both sides, using Cauchy-Schwarz inequality, Assumption~\ref{asm:noise_bias} and the fact that~$(1+x^4)^{1/2}\le 1+x^2,\forall x>0 $ gives
\begin{align*}
&\E[f(\iter_{k+1})|\filtration_k]\\
\le &  f(\iter_{k}) + (\frac{\tilde L}{2}\eta^2 -\eta)\norm{\grad f(\iter_{k})}^2
          + \frac{\tilde L}{2} L_\noise\eta^2 (1+\norm{\iter_k}^2)+0\\
        &+ \frac{\tilde L}{2}\eta^2 \E\Bigl[\norm{\grad f(\iter_{k})}^2(\norm{\iter_{k}}+\norm{\iter_{k+1}})|\filtration_k \Bigr]\\
        &+ \frac{\tilde L}{2}\eta^2 \E^{1/2}\Bigl[\norm{\noise_{k+1}({\iter_k})}^4|\filtration_k\Bigr]\E^{1/2}\Bigl[(\norm{\iter_{k}}+\norm{\iter_{k+1}})^2|\filtration_k \Bigr] \\
       &+0+ \tilde L\eta^2 \E[ \norm{\grad f(\iter_{k})} \norm{\noise_{k+1}(\iter_k)} \norm{\iter_{k+1}} |\filtration_k]\\
      \le&  f(\iter_{k}) + (\frac{\tilde L}{2}\eta^2 -\eta)\norm{\grad f(\iter_{k})}^2
          + \frac{\tilde L}{2} L_\noise \eta^2(1+\norm{\iter_k}^2)\\
        &+ \frac{\tilde L}{2}\eta^2 \E\Bigl[\norm{\grad f(\iter_{k})}^2(\norm{\iter_{k}}+\norm{\iter_{k+1}})|\filtration_k \Bigr]\\
        & + \frac{\tilde L}{2}\eta^2L_\noise ^{1/2} (1+\norm{\iter_k}^2)\E^{1/2}\Bigl[(\norm{\iter_{k}}+\norm{\iter_{k+1}})^2|\filtration_k \Bigr] \\
               &+ \tilde L\eta^2 \E[ \norm{\grad f(\iter_{k})} \norm{\noise_{k+1}(\iter_k)} \norm{\iter_{k+1}} |\filtration_k]\,.
\end{align*}
We then take expectation on both sides. 
For this, we bound the last three terms separately.
Note that the chain starts from the initial distribution~$\sdist$.
By H\"{o}lder's inequality, we have 
\begin{align*}
& \E[\norm{\grad f(\iter_{k})}^2(\norm{\iter_{k}}+\norm{\iter_{k+1}})]\\  
\le &  \E[\norm{\grad f(\iter_{k})}^2\norm{\iter_k}]+\E[\norm{\grad f(\iter_{k})}^2\norm{\iter_{k+1}}]\\
\le & \E^{1/2}[\norm{\grad f(\iter_{k})}^4]\E^{1/2}[\norm{\iter_k}^2]+\E^{1/2}[\norm{\grad f(\iter_{k})}^4]\E^{1/2}[\norm{\iter_{k}}^2]\,.
\end{align*}
By Assumption~\ref{asm:growth} and the fact that $(x+y)^4\le 9(x^4+y^4),\forall x,y\in\R$, we have
\begin{align*}
 \E^{1/2}[\norm{\grad f(\iter_{k})}^4]
 &\le  L^2\E^{1/2}[(1+\norm{\iter_{k}})^4]
 \le 3L^2\sqrt{1+\E[\norm{\iter_k}^4]}\,.
\end{align*}
By Lemma~\ref{bound4mom}, it holds that $\E[\norm{\iter_k}^4]<\mu_{4,\eta},$ 
where the constant $\mu_{4,\eta}$ is defined in Lemma~\ref{bound4mom}.
Moreover, by  Corollary~\ref{bouned2moment}, we also have $\E[\norm{\iter_{k}}^2]\le \mu_{2,\eta},$ 
where the constant $\mu_{2,\eta}$ is defined in Corollary~\ref{bouned2moment}. 
Combining these with previous display gives
\begin{align*}
&  \E[\norm{\grad f(\iter_{k})}^2(\norm{\iter_{k}}+\norm{\iter_{k+1}})]
\le 6L^2\sqrt{1+\mu_{4,\eta}} \sqrt{\mu_{2,\eta}}\,.
\end{align*}
Using the same trick, we obtain
\begin{align*}
&\E \Bigl[ (1+\norm{\iter_k}^2) \E^{1/2}[ (\norm{\iter_k}+\norm{\iter_{k+1}})^2|\filtration_k]\Bigr] \\
\le & \E^{1/2}[(1+\norm{\iter_k}^2)^2] \E^{1/2} \Bigl[\E[ (\norm{\iter_k}+\norm{\iter_{k+1}})^2 |\filtration_k]  \Bigr]\\
\le & \E^{1/2}[2+2\norm{\iter_k}^4] \E^{1/2} \Bigl[\E[ 2\norm{\iter_k}^2+2\norm{\iter_{k+1}}^2 |\filtration_k]  \Bigr]\\
\le & {4}\E^{1/2} [1+\norm{\iter_k}^4]\E^{1/2}[\norm{\iter_k}^2]\\
\le & {4}\sqrt{\mu_{2,\eta}}\sqrt{1+\mu_{4,\eta}}\,.
\end{align*}
By Assumptions~\ref{asm:growth} and~\ref{asm:noise_bias}, we have
\begin{align*}
&\E\Bigl[\E[ \norm{\grad f(\iter_{k})} \norm{\noise_{k+1}(\iter_k)} \norm{\iter_{k+1}} |\filtration_k] \Bigr]\\
\le& \E\Bigl [\norm{\grad f(\iter_k)}\E^{1/2}[\norm{\noise_{k+1}(\iter_k)}^2 |\filtration_k]\E^{1/2}[\norm{\iter_{k+1}}^2|\filtration_k] \Bigr]\\
\le&L_\noise^{1/2} L \E\Bigl [(1+\norm{\iter_k})(1+\norm{\iter_k}^2)^{1/2}\E^{1/2}[\norm{\iter_{k+1}}^2|\filtration_{k}]\Bigr]\\
\le & L L_\noise^{1/2} \E\Bigl [(1+\norm{\iter_k})^2 \E^{1/2}[\norm{\iter_{k+1}}^2|\filtration_k]\Bigr] \\
\le & L L_\noise^{1/2} \E^{1/2} [(1+\norm{\iter_k})^4]  \E^{1/4}[\norm{\iter_k}^4]\\
\le & L L_\noise^{1/2}\sqrt{8+8\mu_{4,\eta}} (\mu_{4,\eta})^{1/4} \\
=&3 L L_\noise^{1/2}(\mu_{4,\eta} +\mu_{4,\eta}^{3/4})\,.
\end{align*}
Collecting pieces then gives 
\begin{align*}
&\E[f(\iter_{k+1})]\\
\le & \, \E[f(\iter_{k})]
+(\frac{\tilde L}{2}\eta^2-\eta)\E[\norm{\grad f(\iter_{k})}^2]
+ \tilde L L_\noise\eta^2(1+\mu_{2,\eta})\\
&+ 3\tilde L L^2 \eta^2 \mu_{2,\eta}^{1/2}\sqrt{1+\mu_{4,\eta}}
+ 2\tilde L L_\noise^{1/2} \eta^2 \mu_{2,\eta}^{1/2} \sqrt{1+\mu_{4,\eta}}
+3 \tilde LL L_\noise^{1/2}\eta^2(\mu_{4,\eta} +\mu_{4,\eta}^{3/4})\\
\le &  \, \E[f(\iter_{k})]
+(\frac{\tilde L}{2}\eta^2-\eta)\E[\norm{\grad f(\iter_{k})}^2] 
+ 12\eta^2\tilde L (L + L_\noise^{1/2}+L_\noise^{1/4})^2\Bigl( 1+\mu_{2,\eta} + \mu_{4,\eta}+\mu_{4,\eta}^{3/4}\Bigr)\,.
\end{align*}
Recall that the iterates~$\chain$ starts from the stationary distribution~$\sdist$ and $\eta< \frac{2}{\tilde L}.$ 
Rearranging the above display gives
\begin{align*}
\E[\norm{\grad f(\iter_{k})}^2] 
\le \frac{2\tilde M \eta}{2-\tilde L\eta} \,,
\end{align*}
where 
\begin{equation*}
\tilde M:=12\tilde L (L + L_\noise^{1/2}+L_\noise^{1/4})^2\Bigl( 1+\mu_{2,\eta}+ \mu_{4,\eta}+\mu_{4,\eta}^{3/4}\Bigr)\,.
\end{equation*}
By Assumption~\ref{asm:loja} and Jensen's inequality,
it holds that 
\begin{align*}
\E[\norm{\grad f(\iter_k)}^2]
\ge & \E[g(f(\iter_k)-f^*)\1 (\norm{\iter_k-\theta^*}\le R) ]+ \gamma\E[(f(\iter_k)]-f^*) \1(\norm{\iter_k-\theta^*}>R)]  \\
\ge & g(\E[(f(\iter_k)]-f^*)\1 (\norm{\iter_k-\theta^*}\le R) ])+ \gamma\E[(f(\iter_k)]-f^*) \1(\norm{\iter_k-\theta^*}>R)] \,.
\end{align*}
Combing this with previous display gives
\begin{align*}
0 \leq \E[(f(\iter_k)]-f^*)\1 (\norm{\iter_k-\theta^*}\le R) ] &\le  g^{-1}\Bigl(\frac{2\tilde M \eta}{2-\tilde L\eta}\Bigr)
\\
0 \leq \E[(f(\iter_k)]-f^*) \1(\norm{\iter_k-\theta^*}>R)]  &\le \frac{2\tilde M \eta}{2-\tilde L\eta}\,.
\end{align*}
This implies
\begin{align*}
0 \leq \sdist(f) -f^*
&=\E[(f(\iter_k)]-f^* \\
&= \E[(f(\iter_k)-f^*)\1 (\norm{\iter_k-\theta^*}\le R)]+\E[(f(\iter_k)-f^*)\1 (\norm{\iter_k-\theta^*}>R)]\\
 &\le  g^{-1}\Bigl( \frac{2\tilde M \eta}{2-\tilde L\eta} \Bigr) +\frac{2\tilde M \eta}{2-\tilde L\eta}\,.
\end{align*}
When the test function~$\phi$ satisfies $\phi=\tilde\phi\circ f$ with the $L_{\tilde\phi}$-Lipschitz function~$\tilde\phi$, we obtain
\begin{align*}
|\sdist(\phi)-\phi(\theta^*)|
\le& L_{\tilde\phi} (\sdist(f)-f^*) \le L_{\tilde\phi}\Big( g^{-1}\Bigl( \frac{2\tilde M \eta}{2-\tilde L\eta} \Bigr) +\frac{2\tilde M \eta}{2-\tilde L\eta} \Big)
\end{align*}
as desired.
\end{proof}

\begin{proof}[Proof of Theorem~\ref{bias_wz_convex}]
Consider the chain~$\chain$ starting from the stationary distirbution~$\sdist.$
By display~\eqref{bias_inq},  it holds that
\begin{equation}
\E[\inprod{\grad f(\iter_k)}{\iter_k-\theta^*}]\le  C_2\eta\,,
\end{equation}
where $C_2$ is a positive constant defined in Theorem~\ref{bias_wz_disp}. 
Note that $f$ is convex, this implies
\begin{equation*}
0 \leq  f(\iter_k)-f^* \le \inprod{\grad f(\iter_k)}{\iter_k-\theta^*}\,.
\end{equation*}
Taking the expectation on both sides and combing this with the previous display gives
\begin{align*}
0 \leq  \sdist(f)-f^* \le C_2\eta\,.
\end{align*}
The desired result readily follows for the test function~$\phi$ satisfies $\phi=\tilde\phi\circ f$ with the $L_{\tilde\phi}$-Lipschitz function~$\tilde\phi.$
\end{proof}

\end{document}